\documentclass[final,12pt]{clear2022} 

\title[Relating Graph Neural Networks to Structural Causal Models]{Relating Graph Neural Networks to Structural Causal Models}
\usepackage{times}

\usepackage{algorithm}
\usepackage{algorithmic}        
\usepackage{graphicx}
\usepackage{adjustbox}
\usepackage{amssymb}
\usepackage{afterpage}
\usepackage{makecell}
\usepackage{blkarray}
\usepackage[shortlabels]{enumitem}
\newtheorem{defin}{Definition}
\newtheorem{prop}{Proposition}
\newtheorem{corollary}{Corollary}
\DeclareMathOperator{\doop}{\textit{do}}
\DeclareMathOperator{\pa}{pa}

\DeclareMathOperator{\tran}{^{\mkern-1.5mu\mathsf{T}}}

\clearauthor{%
 \Name{Matej Zečević} \Email{matej.zecevic@cs.tu-darmstadt.de}\\
 \addr Computer Science Department, TU Darmstadt
 \AND
 \Name{Devendra Singh Dhami} \Email{devendra.dhami@cs.tu-darmstadt.de}\\
 \addr Computer Science Department, TU Darmstadt
 \AND
 \Name{Petar Veličković} \Email{petarv@google.com}\\
 \addr DeepMind, London
 \AND
 \Name{Kristian Kersting} \Email{kersting@cs.tu-darmstadt.de}\\
 \addr Computer Science Department, Centre for Cognitive Science, TU Darmstadt, \\and Hessian Center for AI (hessian.AI)
}

\begin{document}

\maketitle

\vspace{-1.25cm}
\begin{abstract}%
    Causality can be described in terms of a structural causal model (SCM) that carries information on the variables of interest and their mechanistic relations. For most processes of interest the underlying SCM will only be partially observable, thus causal inference tries leveraging the exposed. Graph neural networks (GNN) as universal approximators on structured input pose a viable candidate for causal learning, suggesting a tighter integration with SCM. To this effect we present a theoretical analysis from first principles that establishes a more general view on neural-causal models, revealing several novel connections between GNN and SCM. We establish a new model class for GNN-based causal inference that is necessary and sufficient for causal effect identification. Our empirical illustration on simulations and standard benchmarks validate our theoretical proofs.
\end{abstract}


\setlength{\parskip}{1em}
\vspace{-.5cm}
\section{Introduction}
Understanding causal interactions is central to human cognition and thereby of high value to science, engineering, business, and law \citep{penn2007causal}. Developmental psychology has shown how children explore similar to the manner of scientist, all by asking "What if?" and "Why?" type of questions \citep{gopnik2012scientific, buchsbaum2012power, pearl2018book}, while artificial intelligence research dreams of automating the scientist's manner \citep{mccarthy1998artificial, mccarthy1981some, steinruecken2019automatic}. Deep learning has brought optimizable universality in approximation which refers to the fact that for any function there will exist a neural network that is close in approximation to arbitrary precision \citep{cybenko1989approximation, hornik1991approximation}. This capability has been corroborated by tremendous success in various applications \citep{krizhevsky2012imagenet, mnih2013playing,vaswani2017attention}. Thereby, combining causality with deep learning is of critical importance for research on the verge to a human-level intelligence. Preliminary attempts on a tight integration for so-called neural-causal models \citep{xia2021causal, pawlowski2020deep} exist and show to be promising towards the dream of a system that performs causal inferences at the same scale of effectiveness as modern-day neural modules in their most impressive applications.

\noindent While causality has been thoroughly formalized within the last decade \citep{pearl2009causality,peters2017elements}, deep learning on the other hand saw its success in practical applications with theoretical breakthroughs remaining in the few. \citet{bronstein2017geometric} pioneer the notion of geometric deep learning and an important class of neural networks that follows from the geometric viewpoint and generalize to modern architectures is the graph neural network (GNN) \citep{velivckovic2017graph,kipf2016semi,gilmer2017neural}. Similar to other specialized neural networks, the GNN has resulted in state-of-the-art performance in specialized applications like drug discovery \citep{stokes2020deep} and more recently on ETA prediction in google maps \citep{derrow2021eta}. These specialities, to which we refer to as inductive biases, can leverage otherwise provably impossible inferences \citep{gondal2019transfer}. As the name suggests, the GNN places an inductive bias on the structure of the input i.e., the input's dimensions are related such that they form a graph structure. To link back to causality, at its core lies a Structural Causal Model (SCM) which is considered to be the model of reality responsible for data-generation. The SCM implies a graph structure over its modelled variables, and since GNN work on graphs, a closer inspection on the relation between the two models seems reasonable towards progressing research in neural-causal AI. Instead of taking inspiration from causality's principles for improving machine learning \citep{mitrovic2020representation}, we instead show how GNN can be used to perform causal computations i.e., how causality can emerge within neural models. To be more precise on the term causal inference: we refer to the modelling of Pearl's Causal Hierarchy (PCH) \citep{bareinboim20201on}. That is, we are given partial knowledge on the SCM in the form of e.g.\ the (partial) causal graph and/or data from the different PCH-levels.

\noindent Overall, we make a number of key contributions:\ (1) We derive, from first principles, a theoretical connection between GNN and SCM; (2) We define a more fine-grained NCM; (3) We formalize interventions for GNN and by this establish a new neural-causal model class that makes use of auto-encoders; (4) We provide theoretical results and proofs on the feasibility, expressivity, and identifiability of this new model class while relating to existing work (5) We empirically examine our theoretical model for practical causal inference on identification and estimation tasks. We make our code publicly available: \url{https://anonymous.4open.science/r/Relating-Graph-Neural-Networks-to-Structural-Causal-Models-A8EE}.

\section{Background and Related Work}
Before presenting our main theoretical findings, we briefly review the background on variational methods for generative modelling, on graph neural networks as non-parametric function approximator that leverage structural information, and conclusively on causal inference through the process of intervention/mutilation.

\noindent{\bf Notation.} We denote indices by lower-case letters, functions by the general form $g(\cdot)$, scalars or random variables interchangeably by upper-case letters, vectors, matrices and tensors with different boldface font $\mathbf{v}, \mathbf{V}, \mathbf{\mathsf{V}}$ respectively, and probabilities of a set of random variables $\mathbf{X}$ as $p(\mathbf{X})$. Pearl's Causal Hierarchy (PCH) is denoted with $\mathcal{L}_i, i\in\{1,2,3\}$, and an intervention via the $\doop$-operator.

\noindent{\bf Variational Inference.} Similar to the notions of disentanglement and causality, latent variable models propose the existence of apriori unknown variables $\mathbf{Z}$ to jointly model the phenomenon of interest with observed data, $p(\mathbf{X}, \mathbf{Z})$. The Variational Inference (VI) technique makes use of optimization, as an alternative to Markov chain Monte Carlo sampling (MCMC) approaches, for overcoming the curse of dimensionality\footnote{Uniformly covering a unit hypercube of $n$ dimensions with $k$ samples scales exponentially, $O(k^n)$.} when estimating probability distributions \citep{jordan1999introduction,blei2017variational}. In this Bayesian setting, the inference problem amounts to estimating the latent variable conditional $p(\mathbf{Z}\mid \mathbf{X})$ through the closest density of a pre-specified family $\mathcal{Q}$, that is, 
\begin{equation}\label{eq:vi}
    q^*(Z) = \arg\min_{q\in\mathcal{Q}} \text{KL}(q(\mathbf{Z})\mid\mid p(\mathbf{Z}\mid \mathbf{X}))
\end{equation} where the distance measure is set to be the Kullback-Leibler divergence. Inspecting Bayes Rule exposes that $p(\mathbf{Z}\mid \mathbf{X})=\frac{p(\mathbf{X}, \mathbf{Z})}{p(\mathbf{X})}$ where the evidence in the denominator is an exponential term in $\mathbf{Z}$, that is $p(\mathbf{X})=\int p(\mathbf{X}, \mathbf{Z}) \,d\mathbf{Z}$, thus rendering the overall problem described in Eq.\ref{eq:vi} intractable in the average case. Originally derived using Jensen's inequality \citep{jordan1999introduction}, a tractable lower bound on the evidence is revealed,
\begin{align}\label{eq:elbo}
\begin{split}
    \log p(\mathbf{X}) - \text{KL}(q(\mathbf{Z})\mid\mid p(\mathbf{Z}\mid \mathbf{X})) =\\ \mathbb{E}_q[\log p(\mathbf{X}\mid \mathbf{Z})] - \text{KL}(q(\mathbf{Z})\mid\mid p(\mathbf{Z}))
\end{split}
\end{align} where the first term expresses likelihood (or reconstruction) of the data under the given parameters while the divergence terms counteracts such parameterization to adjust for the assumed prior. Choosing $p_{\pmb{\phi}}(\mathbf{X}\mid\mathbf{Z})$ and $q(\mathbf{Z}){:=}q_{\pmb{\theta}}(\mathbf{Z}\mid \mathbf{X})$ to be parameterized as neural networks leads to the variational auto-encoder (VAE) model class \citep{kingma2019introduction}. Importance sampling \citep{rubinstein2016simulation} reveals a connection between variational methods (VAE) and sampling techniques for performing marginal inference i.e., since
\begin{equation}\label{eq:impsamp}
    p(\mathbf{X})\approx \frac{1}{n} \sum_{i=1}^n \frac{p_{\pmb{\phi}}(\mathbf{X}\mid\mathbf{z}_i)p(\mathbf{z}_i)}{q_{\pmb{\theta}}(\mathbf{z}_i\mid \mathbf{X})}
\end{equation} where the number of samples $n$ is being kept moderate through the likelihood ratio induced by $q$.

\noindent{\bf Graph Neural Networks.} In geometric deep learning, as portrayed by \citep{bronstein2021geometric}, graph neural networks (GNN) constitute a fundamental class of function approximator that place an inductive bias on the structural relations of the input. A GNN layer $f(\mathbf{D},\mathbf{A}_G)$ over some data considered to be vector-valued samples of our variables $\{\mathbf{d}_i\}^n_{i=1}\mathbf{D}\in R^{d\times n}$ and an adjacency representation $\mathbf{A}_G \in [0,1]^{d\times d}$ of a graph $G$ is generally considered to be a permutation equivariant\footnote{That is, for some permutation matrix $\mathbf{P}\in [0,1]^{d\times d}$, it holds that $f(\mathbf{P}\mathbf{D},\mathbf{P}\mathbf{A}_G\mathbf{P}\tran) = \mathbf{P}f(\mathbf{D},\mathbf{A}_G)$.} application of permutation invariant functions $\phi(\mathbf{d}_X, \mathbf{D}_{\mathcal{N}^G_X})$ on each of the variables (features) $\mathbf{d}_i$ and their respective neighborhoods within the graph $\mathcal{N}^G_i$. The most general form of a GNN layer is specified by 
\begin{equation}\label{eq:gnn}
    \mathbf{h}_i = \phi\bigg(\mathbf{d}_i, \bigoplus_{j\in\mathcal{N}^G_i} \psi(\mathbf{d}_i, \mathbf{d}_j)\bigg),
\end{equation}
where $\mathbf{h}_i$ represents the updated information of node $i$ aggregated ($\bigoplus$) over its neighborhood in the form of messages $\psi$. The flavour of GNN presented in Eq.\ref{eq:gnn} is being referred to as message-passing \citep{gilmer2017neural} and constitutes the most general class of GNN that supersets both convolutional \citep{kipf2016semi} and attentional \citep{velivckovic2017graph} flavours of GNN. In the context of representation learning on graphs, GCN were previously used within a VAE pipeline as means of parameterization to the latent variable posterior $p(\mathbf{Z}\mid \mathbf{X})$ \citep{kipf2016variational}.

\noindent{\bf Causal Inference.} A (Markovian) Structural Causal Model (SCM) as defined by \cite{pearl2009causality,peters2017elements} is specified as $\mathfrak{C}:=(\mathbf{S},P(\mathbf{U}))$ where $P(\mathbf{U})$ is a product distribution over exogenous unmodelled variables and $\mathbf{S}$ is defined to be a set of $d$ structural equations
\begin{align}\label{eq:scm}
V_i := f_i(\pa(V_i),U_i), \quad \text{where} \ i=1,\ldots,d
\end{align}
with $\pa(V_i)$ representing the parents of variable $V_i$ in graph $G(\mathfrak{C})$. An intervention $\doop(\mathbf{W}), \mathbf{W} {\subset} \mathbf{V}$ on a SCM $\mathfrak{C}$ as defined in (\ref{eq:scm}) occurs when (multiple) structural equations are being replaced through new non-parametric functions $g_{\mathbf{W}}$ thus effectively creating an alternate SCM $\mathfrak{C}_2:=\mathfrak{C}^{\doop(\mathbf{W}=g_{\mathbf{W}})}$. Interventions are referred to as \textit{imperfect} if the parental relation is kept intact, $g_i(\pa_i,\cdot)$, and as \textit{atomic} if $g_i = a$ for $a \in \mathbb{R}$. An important property of interventions often referred to as "modularity" or "autonomy"\footnote{See Section 6.6 in \citep{peters2017elements}.} states that interventions are fundamentally of local nature, formally
\begin{align} \label{eq:autonomy}
p^{\mathfrak{C}_1}(V_i \mid \pa(V_i)) = p^{\mathfrak{C}_2}(V_i \mid \pa(V_i))\;,
\end{align}
where the intervention of $\mathfrak{C}_2$ occurred on variable $V_j$ opposed to $V_i$. This suggests that mechanisms remain invariant to changes in other mechanisms which implies that only information about the effective changes induced by the intervention need to be compensated for. An important consequence of autonomy is the truncated factorization
\begin{align} \label{eq:truncatedfactorization}
p(\mathbf{V}) = \prod\nolimits_{V\notin \mathbf{W}} p(V\mid \pa(V)))
\end{align}
derived by \cite{pearl2009causality}, which suggests that an intervention $\doop(\mathbf{W})$ introduces an independence of a set of intervened nodes $\mathbf{W}$ to its causal parents. 
Another important assumption in causality is that causal mechanisms do not change through intervention suggesting a notion of invariance to the cause-effect relations of variables which further implies an invariance to the origin of the mechanism i.e., whether it occurs naturally or through means of intervention \citep{pearl2016causal}. A SCM $\mathfrak{C}$ is capable of emitting various mathematical objects such as graph structure, statistical and causal quantities placing it at the heart of causal inference, rendering it applicable to machine learning applications in marketing \citep{hair2021data}), healthcare \citep{bica2020time}) and education \citep{hoiles2016bounded}. A SCM induces a causal graph $G$, an observational/associational distribution $p^{\mathfrak{C}}$, can be intervened upon using the $\doop$-operator and thus generate interventional distributions $p^{\mathfrak{C};\doop(...)}$ and given some observations $\mathbf{v}$ can also be queried for interventions within a system with fixed noise terms amounting to counterfactual distributions $p^{\mathfrak{C\mid \mathbf{V}=\mathbf{v}};\doop(...)}$. As suggested by the Causal Hierarchy Theorem (CHT) \citep{bareinboim20201on}, these properties of an SCM almost always form the Pearl Causal Hierarchy (PCH) consisting of different levels of distributions being $\mathcal{L}_1$ associational, $\mathcal{L}_2$ interventional and $\mathcal{L}_3$ counterfactual. This hierarchy suggests that causal quantities ($\mathcal{L}_i,i\in\{2,3\}$) are in fact richer in information than statistical quantities ($\mathcal{L}_1$), and the necessity of causal information (e.g.\ structural knowledge) for inference based on lower rungs e.g.\ $\mathcal{L}_1 \not\rightarrow \mathcal{L}_2$. Finally, to query for samples of a given SCM, the structural equations are being simulated sequentially following the underlying causal structure starting from independent, exogenous variables $U_i$ and then moving along the causal hierarchy of endogenous variables $\mathbf{V}$. To conclude, consider the formal definition of valuations for the first two layers being
\begin{equation} \label{eq:l12val}
    p^{\mathfrak{C}}(\mathbf{y}\mid \doop(\mathbf{x})) = \sum_{\{\mathbf{u}\mid \mathbf{Y}_{\mathbf{x}}(\mathbf{u})=\mathbf{y}\}} p(\mathbf{u})
\end{equation} for instantiations $\mathbf{x},\mathbf{y}$ of the node sets $\mathbf{X},\mathbf{Y} \subseteq \mathbf{V}$ where $\mathbf{Y}_{\mathbf{x}}: \mathbf{U}\mapsto \mathbf{Y}$ denotes the value of $\mathbf{Y}$ under intervention ${\mathbf{x}}$.

\setlength{\parskip}{0em}
\section{The GNN-SCM-NCM Connection}
To expand further on the boundaries of the integration between causality and machine learning, we perform a theoretical investigation on the relation between graph neural networks (GNN) and structural causal models (SCM), thereby transitively also to neural causal models (NCM). While all the established results on causal identification have proven that intervention/manipulation is not necessary for performing causal inference, the concept of intervention/manipulation still lies at the core of causality as suggested by the long-standing motto of Peter Holland and Don Rubin \textit{'No causation without manipulation'} \citep{holland1986statistics}. The centrality of interventions is why we choose to consider them as a starting point of our theoretical investigation. To this effect, we first define a process of intervention within the GNN computation layer that will subsequently reveal sensible properties of the process akin to those of intervention on SCM. 
\begin{defin}\label{def:iGNN}
\textbf{(Interventions within GNN.)} An intervention $\mathbf{x}$ on the corresponding set of variables $\mathbf{X}\subseteq \mathbf{V}$ within a GNN layer $f(\mathbf{D},\mathbf{A}_G)$, denoted by $f(\mathbf{D},\mathbf{A}_G{\mid} \doop(\mathbf{X}=\mathbf{x}))$, is defined as a modified layer computation, \begin{equation}\label{eq:iGNNrule}
    \mathbf{h}_i = \phi\bigg(\mathbf{d}_i, \bigoplus_{j\in\mathcal{M}^G_i} \psi(\mathbf{d}_i, \mathbf{d}_j)\bigg),
\end{equation} where the intervened local neighborhood is given by
\begin{equation}
    \mathcal{M}^G_i = \{j \mid j \in \mathcal{N}^G_i, j\not\in \pa_i {\iff} i\in\mathbf{X} \}
\end{equation} where $\mathcal{N}^G$ denotes the regular graph neighborhood. Such GNN-layers are said to be \textit{interventional}.
\end{defin}
An intervention, just like in an SCM, is of local nature i.e., the new neighborhood of a given node is a subset of the original neighborhood at any time, $\mathcal{M}\subseteq \mathcal{N}$. The notion of intervention belongs to the causal layers of the PCH i.e., layers 2 (interventional) and 3 (counterfactual). Fig.\ref{fig:iGNN} presents an intuitive illustration each for both the underlying SCM with its various properties and the intervention process within the GNN layer.
\begin{figure*}[t]
\centering
\includegraphics[width=0.9\textwidth]{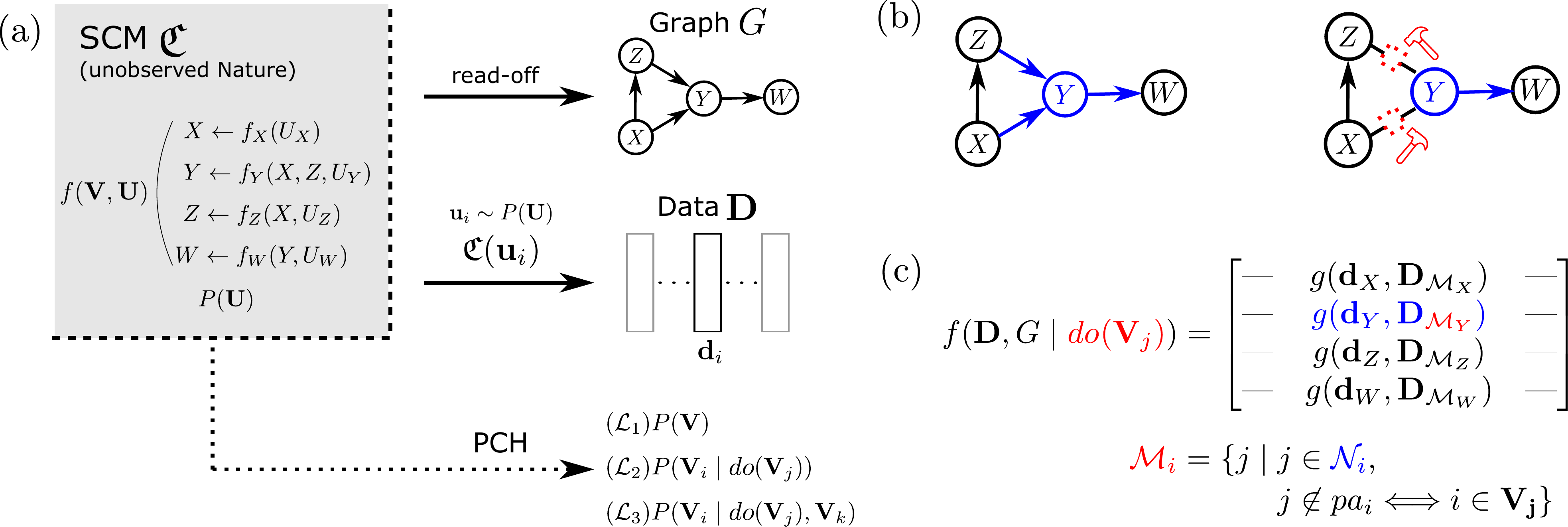}
\caption{\textbf{Graph Neural Networks and Interventions akin to SCM.} A schematic overview. (a) shows the unobserved SCM $\mathfrak{C}$ that generates data $\mathbf{D}$ through instantiations of the exogenous variables $\mathbf{U}$. The SCM implies a graph structure $G$ and the PCH ($\mathcal{L}_i$) (b) shows the intuition behind interventions using the $\doop$-operation (c) presents the mathematical formalism (Def.\ref{def:iGNN}) highlighting the intervention (red) and the regular neighborhood (blue). (Best viewed in color.)}
\label{fig:iGNN}
\end{figure*}
The motivational origin of this work lies in the tighter integration of causality with today's machine learning methodologies, more specifically neural network variants. We envision a fully-differentiable system that combines the benefits of both worlds. As a step towards this goal, we introduce the concept of intervention for GNNs (Def.\ref{def:iGNN}). The reader might wonder why counterfactuals ($\mathcal{L}_3$) are not being covered in this work. The reason for this lies in the fact that a conversion between GNN and SCM will necessarily have to cope with transforming a shared, global function $\psi$ into the collection of all local partial mechanisms $f_{ij}$ of any structural equation. Thereby, optimization becomes tremendously difficult. More formally, we state the following theorem on the model conversion.
\begin{theorem}\label{thm:gnnscm}
\textbf{(GNN-SCM Conversion.)} Consider the most general formulation of a message-passing GNN node computation $\mathbf{h}_i{:} \mathcal{F}\mapsto \mathcal{F}^{\prime}$ as in Eq.\ref{eq:gnn}. For any SCM $\mathfrak{C}{=}(\mathbf{S},P(\mathbf{U}))$ there exists always a choice of feature spaces $\mathcal{F},\mathcal{F}^{\prime}$ and shared functions $\phi,\psi$, such that for all structural equations $f{\in}\mathbf{S}$ it holds that $\mathbf{h}_i = f_i$.
\end{theorem}
\begin{proof}
Compact (details in Appendix). Let $f_i(\pa(i),U_i)=f_{i}(U_i, \mathcal{A}_i)+\sum_{j\in\pa(i)} f_{ij}(V_j)$ be a structural equation ($\mathcal{A}_i\in 2^{|\pa_i|},f_i\in\mathbf{S}, \mathfrak{C}{=}(\mathbf{S},P(\mathbf{U}))$) and its scalar-decomposition following Thm.1 in \citep{kuo2010decompositions}. The following mapping:
\begin{align}
\mathcal{F} = \mathbf{V}\cup\mathbf{U} = \mathcal{F}^{\prime}, \quad 
\phi(i,\dots) = f_{U_X}(U_X,\mathcal{A}_i)+ \sum \dots \quad
\psi(i,j) = f_{ij}
\end{align} where ($\dots$) is the remainder of the GNN-computation (Eq.\ref{eq:gnn}), defines a general construction scheme:
\begin{align}
\begin{split}
\mathbf{h}_i &= \phi\bigg(\mathbf{d}_i, \textstyle\bigoplus_{j\in\mathcal{N}^G_i} \psi(\mathbf{d}_i, \mathbf{d}_j)\bigg) = f_{U_i}(U_i,\mathcal{A}_i) + \sum_{j\in\pa(i)} f_{ij}(V_j) = f_i.
\end{split}
\end{align}
\end{proof}
The common ground between SCM and GNN lies within the assumed graph structure and thus is deemed suitable as a starting point for a reparameterization from SCM to GNN as Thm.\ref{thm:gnnscm} suggests. However, while Thm.\ref{thm:gnnscm} is powerful in the sense that any GNN can be seen as a neural SCM variant, the theorem does not give away any information on optimization. It follows naturally that $\psi$ is a \textit{shared} function amongst all nodes of the graph while an SCM considers a specific mechanism \textit{for each} of the nodes in the graph, and thus optimization becomes difficult. In a nutshell, the messages $\psi(i,j)$ need to model each of the dependency terms $f_{ij}$ within a structural equation, such that the messages themselves become a descriptor of the causal relation for $(i\leftarrow j)$. Nonetheless, the theoretical connection's existence suggests tighter integration for NCM with an important consequence being the connection to the base-NCM definition (see \citet{xia2021causal}).
\begin{corollary}\label{cor:ncmgnn}
\textbf{(NCM-Type 2.)} Allowing for the violation of sharedness of $\psi$ as depicted in Thm.\ref{thm:gnnscm} and choosing $\mathcal{F}=\mathcal{F}^{\prime}=\mathbf{U}\cup\mathbf{V}$ to be the union over endo- and exogenous variables, $\phi(i,\dots)=f_{U_i}(U_i,\mathcal{A}_i)+\sum(\dots)$ to be a sum-aggregator with noise term selection with $\mathcal{A}_i{\in} 2^{|\pa_i|}$, and $\psi=\{f^{ij}_{\pmb{\theta}}\}$ to be the dependency terms of the structural equations $f_i$ modelled as feedforward neural networks. Then the computation layer $\{\mathbf{h}_i\}^{|V|}_i$ is a special case of the NCM as in \citep{xia2021causal}.
\end{corollary}
Because of space restrictions we provide the proof to Cor.\ref{cor:ncmgnn} and all subsequent mathematical results within the supplementary section. To be more precise, the NCM-Type 2 portrayed in Cor.\ref{cor:ncmgnn} is more fine-grained than the the definition of NCM in \citep{xia2021causal} since their formulation models structural equations using feedforward nets ($|V|$) while the NCM-Type 2 additionally models the \textit{dependency terms within} each of the structural equations ($|\mathcal{E}^2|$). Fig.\ref{fig:ncmt2} provides a schematic illustration of the discussed concepts, that is, both for the GNN to SCM conversion from Thm.\ref{thm:gnnscm} and the NCM-Type 2 comparison to regular NCM from Cor.\ref{cor:ncmgnn}.
\begin{figure}[!t]
\centering
\includegraphics[width=0.7\columnwidth]{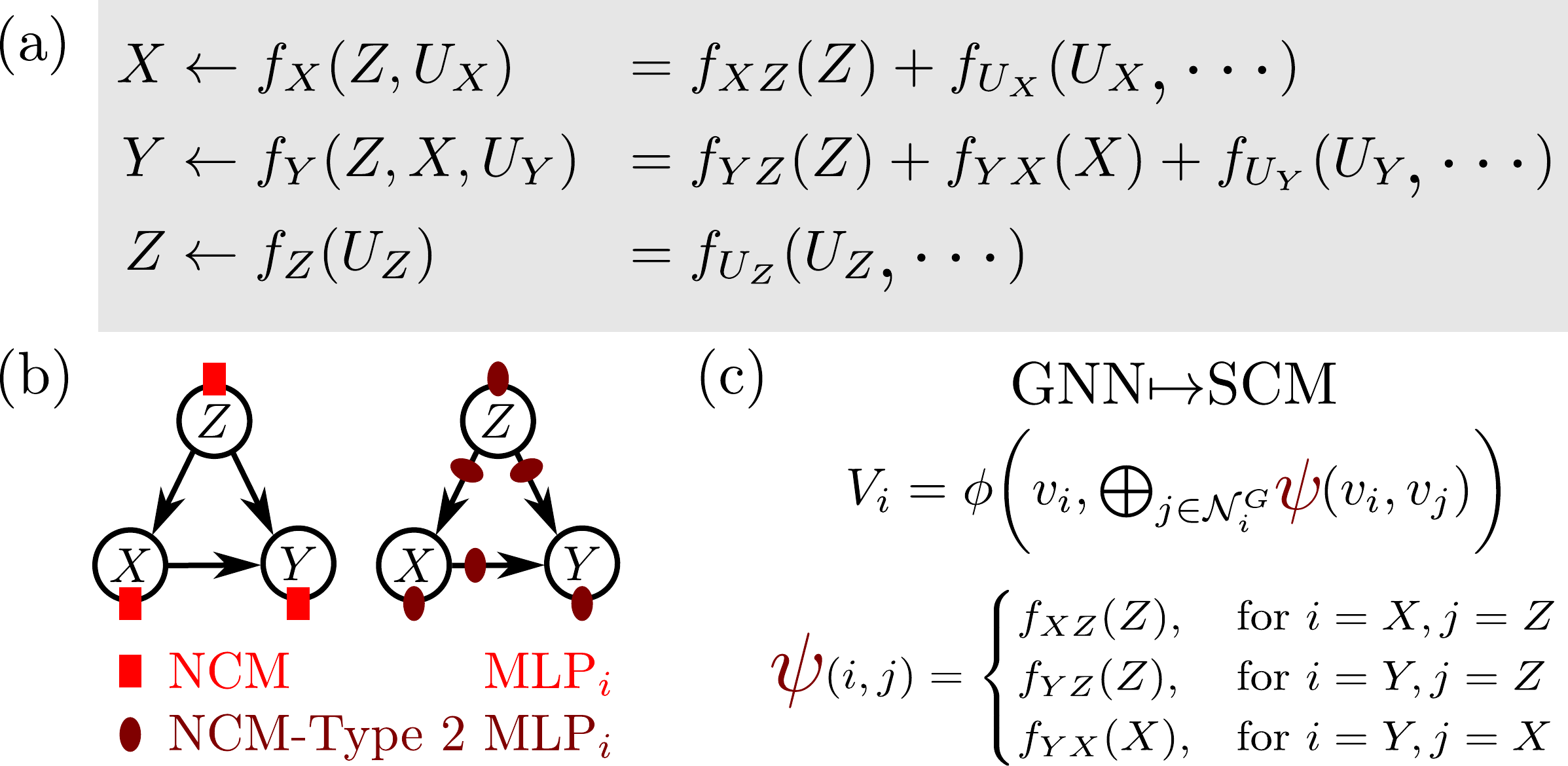}
\caption{\textbf{NCM-Type 2 and the GNN-SCM Conversion.} A schematic overview of the results established in Thm.\ref{thm:gnnscm} and Cor.\ref{cor:ncmgnn} (a) shows an example SCM as the underlying reality with its structural equations and their decompositions into univariate parent dependency-terms $f_{ij}$ alongside a potentially non-empty $(\mathcal{A}_i\neq \emptyset, \mathcal{A}_i{\in}2^{|\pa_i|})$ joining function $f_{U_i}(U_i, \mathcal{A}_i)$ (b) shows the NCM as defined by \citep{xia2021causal} which models on node-level opposed to the edge-level as for NCM-Type 2 (Cor.\ref{cor:ncmgnn}) (c) shows an example of an infeasible $\psi$ function that has to model all relevant dependency terms. (Best viewed in color.)}
\label{fig:ncmt2}
\end{figure} Again, the decomposition (or fine-grained view) in Fig.\ref{fig:ncmt2}(a) follows from \citep{kuo2010decompositions}. To illustrate, consider the following example:
\begin{align}
    &\mathfrak{C} = (\{f_X(Z,U_X):=Z\wedge U_X, \ f_Z(U_Z):=U_Z\}, \quad P(U_X, U_Z)), \\
&f_X = \begin{blockarray}{ccc}
U_X=0 & U_X=1\\
\begin{block}{(cc)c}
  0 & 1 & Z=1 \\
  0 & 0 & Z=0 \\
\end{block}
\end{blockarray} \iff \begin{aligned}
        f_{XZ}(Z) &+ f_{U_X}(U_X, Z)\\
        [Z] &+ [U_X - (Z\vee U_X)].
    \end{aligned} \label{eq:example}
\end{align}
The decomposition in (\ref{eq:example}) simply takes the identity of the parent ($Z$) while considering the negated logical OR for the remainder term $f_{U_X}$. Its sum then results in the original logical AND function $f_X$ of $\mathfrak{C}$. While the decomposition for this specific example does not seem to reveal any sensible advantage, it does for other examples more easily, e.g. when $f_X$ is a linear function of its arguments then the term $f_{U_X}$ will only depend on the unmodelled term $U_X$ and not the other variables (i.e., the argument list is empty $\mathcal{A}=\emptyset$). The advantage becomes evident in that we have a more fine-grained view onto NCM. On a side note, it is worth noting that the computation layer in Cor.\ref{cor:ncmgnn}, as soon as sharedness is being violated, generally is not considered to be a GNN layer anymore since sharedness just like adjacency information and permutation-equivariance belong to key properties of the GNN definition. Nonetheless, Cor.\ref{cor:ncmgnn} predicts a new type of NCM that should be investigated in more detail in future research as it should provide the same theoretical guarantees as regular NCM, while being more flexible and interpretable because of the fine-grained modelling scope (Fig.\ref{fig:ncmt2}(b)). 

\section{GNN-based Causal Inference}
Upon observing the difficulty of optimization for a general causal model based on GNN as suggested by both the variant in Thm.\ref{thm:gnnscm} (Fig.\ref{fig:ncmt2}(c)) and the NCM-variant in Cor.\ref{cor:ncmgnn} (Fig.\ref{fig:ncmt2}(b)), we decide to focus our theoretical investigation on what is reasonably attainable within the GNN framework since the inductive bias on graphs still seems appealing. The graph-intersection of these two model classes can in fact be leveraged to partially represent the PCH as we show next. I.e., up to rung 2 of interventional queries ($\mathcal{L}_2$). Instead of turning a GNN into an SCM (as in Thm.\ref{thm:gnnscm}), we will consider the reverse direction. Thereby, we define a GNN construction based on an SCM.
\begin{defin}\label{def:scmgnn}
\textbf{($\mathfrak{G}$-GNN construction.)} Let $\mathfrak{G}$ be the graph induced by SCM $\mathfrak{C}$. A GNN layer $f(\mathbf{D},\mathbf{A}_G)$ for which $G=\mathfrak{G}$ is said to be $\mathfrak{G}$-constructed.
\end{defin}
We believe that $\mathfrak{G}$-GNN (Def.\ref{def:scmgnn}) make an important step towards connecting the graph structure induced by an SCM with the structured computation within a GNN layer. Developing further on this notion in future research might allow for neural-graph models that are inherently causal\footnote{Consider \citep{xia2021causal} for a treatise on Neural Causal Model (NCM) using sets of feed-forward neural networks.}.
\begin{prop}\label{prop:equiv}
\textbf{(Graph Mutilation Equivalence.)} Let $\mathfrak{C}$ be an SCM with graph $\mathfrak{G}$ and let $f$ be a $\mathfrak{G}$-GNN layer. An intervention $\doop(\mathbf{X}), \mathbf{X}\subseteq V,$ on both $\mathfrak{C}$ and $f$ produces the same mutilated graph.
\end{prop}
It is worth noting that an intervention within a GNN layer (Def.\ref{def:iGNN}) is dual to the notion of intervention within an SCM i.e., like observing within the mutilated graph is akin to interventions, computing a regular GNN layer on the mutilated graph is akin to intervening on the original GNN layer, that is, $f(\mathbf{D},\mathbf{A}_G{\mid} \doop(\mathbf{X}))=f(\mathbf{D},\mathbf{A}_{G^{\prime}})$ where $G^{\prime}$ is the graph upon intervention $\doop(\mathbf{X})$.
In the spirit of \citep{xia2021causal}, we define what it means for an SCM to be consistent with another causal model in terms of the PCH layers $\mathcal{L}_i$.
\begin{defin}\label{def:consistency}
\textbf{(Partial $\mathcal{L}_i$-Consistency.)} Consider a model $\mathcal{M}$ capable of partially emitting PCH, that is $\mathcal{L}_i$ for $i{\in}\{1,2\}$, and an SCM $\mathfrak{C}$. $\mathcal{M}$ is said to be partially $\mathcal{L}_i$-consistent w.r.t.\ $\mathfrak{C}$ if $\mathcal{L}_i(\mathcal{M})\subset\mathcal{L}_i(\mathfrak{C})$ with $|\mathcal{L}_i(\mathcal{M})|>0$.
\end{defin}
In the following we will usually omit the word partial since Def.\ref{def:consistency} is the only consistency employed. A causal model most generally defined can be considered to be something that can carry out causal inferences and in Def.\ref{def:consistency} we consider such models that are capable of emitting the $\mathcal{L}_i$-distributions. The definition proposes that an SCM can match with such a model, if it agrees on a subset of all conceivable distributions for a level $i$, which for $|\mathcal{L}_1|=1$ but $|\mathcal{L}_2|\rightarrow\infty$.
In their work on variational graph auto-encoders (VGAE) \citep{kipf2016variational} considered the application of the GNN model class to the posterior parameterization, $q(\mathbf{Z}\mid\mathbf{D},\mathbf{A})=f(\mathbf{D},\mathbf{A})$. However, their main incentive posed to be a generative model on graphs themselves i.e., $p(\mathbf{A}\mid\mathbf{Z})$. In the following, we define VGAE on the original data as traditionally done within vanilla VAE models that don't consider (or ignore) structured input.
\begin{defin}\label{def:dVGAE}
\textbf{(Variational Graph Auto-Encoder.)} We define a VGAE $\mathcal{V}{=}(q(\mathbf{Z}{\mid}\mathbf{D}), p(\mathbf{D}{\mid}\mathbf{Z}))$, as a data-generative model, with inference and generator model respectively such that $q{:=}f_{\pmb{\theta}}(\mathbf{D},\mathbf{A})$ is a GNN layer and $p{:=}g_{\pmb{\phi}}(\mathbf{D},\mathbf{Z})$ some parameterized data $\mathbf{D}$ dependent model, where $\pmb{\theta},\pmb{\phi}$ are the variational parameters.
\end{defin}
Since SCM induce a graph structure dictated by the set of structural equations that implicitly govern the structural relations between the modelled variables, it is important to define a density estimator which acts on the variable space opposed to the space of plausible graph structures (Def.\ref{def:dVGAE}). Using a GNN layer for modelling the original data does not compromise on the expressiveness of the model class as the following theorem suggests.
\begin{theorem}\label{thm:uda}
\textbf{(Universal Density Approximation.)} There exists a latent variable family $q$ and a corresponding data-generative VGAE $\mathcal{V}{=}(q,p)$ such that the modelled density can approximate any smooth density at any nonzero error.
\end{theorem}
This theorem suggests that the defined VGAE class, for a suitable choice of latent variable model, is a universal approximator of densities (UDA). A suitable choice for $q$ in Thm.\ref{thm:uda} is the mixture of Gaussians family \citep{goodfellow2014explaining, plataniotis2017gaussian}. Since causal estimates are determined by their theoretical identifiability from the partially available information on the true underlying SCM $\mathfrak{C}^*$ (e.g.\ structural constraints), UDA properties of the estimator are beneficial. Since the data-generative VGAE is a probabilistic model capable of approximating distributions to arbitrary precision, we are finally ready to define a causal model based on GNN layers by introducing \textit{interventional} GNN layers (Def.\ref{def:iGNN}) into said VGAE class.
\begin{defin}\label{def:iVGAE}
\textbf{(Interventional VGAE.)} An interventional VGAE is a data-generative VGAE $\mathcal{V}{=}(q,p)$ where both $q,p$ are set to be interventional GNN layers $f_i(\mathcal{D}_i,\mathbf{A}_G{\mid} \doop(\mathbf{X}))$ with $\mathcal{D}_i{\in}\{\mathbf{D},\mathbf{Z}\}$.
\end{defin}
While the reader might expect the data-generative VAE model to make use of the interventions within GNN layers from Def.\ref{def:iGNN}, what might come as surprise is the usage of a second (interventional) GNN layer as decoder because the second layer parameterization considers the reverse module $p(\cdot{\mid} \mathbf{Z})$ which takes the latent variable $\mathbf{Z}$ which only transitively via $\mathbf{D}$ acquires its connection to the adjacency $\mathbf{A}_G$. Interventional capability within both models is beneficial since the knowledge on the specific mutilation becomes available throughout the complete model. Modified decoder that leverage pre-defined inductive biases have been deployed in the literature by for instance \citep{kipf2018neural}. Fig.\ref{fig:iVGAE} illustrates this intuition and parallels it with the computation within an SCM. The interventional VGAE (iVGAE) is a causal model, like an SCM, capable of emitting $\mathcal{L}_i$-distributions from the PCH. With this, we are able to provide statements on the causal capabilities of this extended model class. More precisely, following the mathematical ideas of \citep{xia2021causal} we show that iVGAE overlap with the space of all SCMs $\Omega$.
\begin{figure}[t]
\centering
\includegraphics[width=0.9\columnwidth]{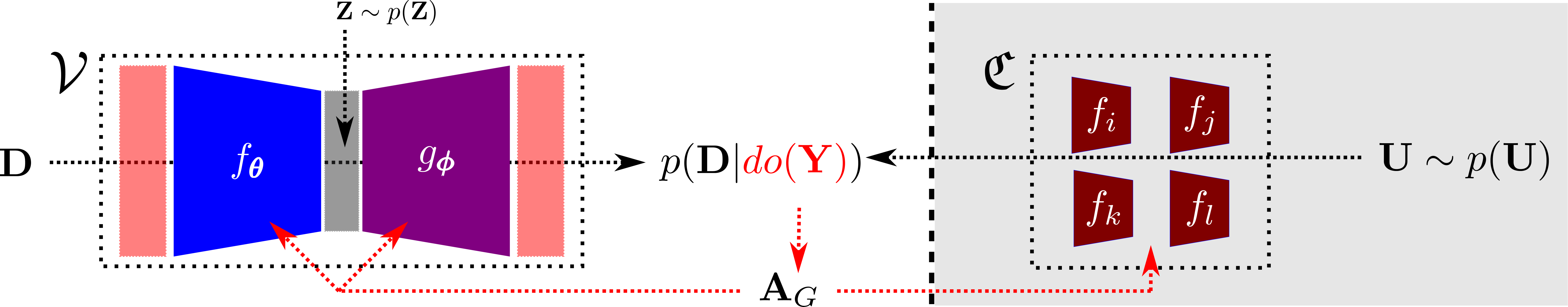}
\caption{\textbf{Neural-Causal Model based on Interventional GNN Layers.} A schematic overview of the inference process within the neural-causal iVGAE $\mathcal{V}$ model and its similarities to the same process within an SCM $\mathcal{C}$. While $\mathfrak{C}$ can be directly queried for the causal effect $p(\mathbf{D}{\mid}\doop(\mathbf{Y}))$ through valuations $\mathbf{U}$, the iVGAE makes use of corresponding data $\mathbf{D}$. (Best viewed in color.)}
\label{fig:iVGAE}
\end{figure}
\begin{theorem}\label{thm:express}
\textbf{(Expressivity.)} For any SCM $\mathfrak{C}$ there exists an iVGAE $\mathcal{V}(\pmb{\theta},\pmb{\phi})$ for which $\mathcal{V}$ is $\mathcal{L}_2$-consistent w.r.t\ $\mathfrak{C}$.
\end{theorem}
Similar to the peculiar specifications of NCM \citep{xia2021causal}, an iVGAE does not loose in expressivity to SCM regarding the first two rungs of the PCH. Thus it can also act as a proxy for causal inferences. An important consequence of the CHT \citep{bareinboim20201on}, which also applies to iVGAE is that the PCH is kept in tact across scenarios.
\begin{corollary}\label{cor:cht}
\textbf{(iVGAE Partial Causal Hierarchy Theorem.)} Consider the sets of all SCM and iVGAE, $\Omega,\Upsilon$, respectively. If for all $\mathcal{V}\in\Upsilon$ it holds that $\mathcal{V}$ is $\mathcal{L}^p_1(\mathcal{V}){=}\mathcal{L}^p_1(\mathfrak{C})\implies\mathcal{L}^q_2(\mathcal{V}){=}\mathcal{L}^q_2(\mathfrak{C})$ with $\mathfrak{C}\in\Omega$, where $\mathcal{L}^p_i \subset \mathcal{L}_i$ is a selection $p,q$ over the set of all level 1,2 distributions respectively, then we say that layer 2 of iVGAE collapses relative to $\mathfrak{C}$. On the Lebesgue measure over SCMs, the subset in which layer 2 of iVGAE collapses to layer 1 has measure zero. 
\end{corollary}
As previously pointed out by \citep{xia2021causal, bareinboim20201on}, the CHT does not impose a negative result by claiming impossibility on low-to-high-layer inferences, however, it suggests that even sufficient expressivity of a model does not allow for the model to overcome the boundaries of the layers (unless causal information e.g.\ in the form of structural constraints is available). Another noteworthy consequence of both Thms.\ref{thm:express} and \ref{thm:gnnscm} is the incapability of handling counterfactuals.
\begin{corollary}\label{cor:l3}
\textbf{(iVGAE Limitation.)} For any SCM $\mathfrak{C}$ there exists no iVGAE $\mathcal{V}$ such that $\mathcal{V}$ is $\mathcal{L}_3$-consistent w.r.t.\ $\mathfrak{C}$.
\end{corollary}
While the formulation of an iVGAE restricts itself from the full expressivity of the PCH (like SCM or NCM), there are no restrictions on the lower causal layers in addition to a more compact overall model. In Def.\ref{def:scmgnn} we pointed to the relation between SCM and GNN via the deployed graph. The following theorem states that if $\mathfrak{G}$-GNN layers are being deployed within a corresponding VGAE, then we can have consistency to any SCM of choice.
\begin{corollary}\label{cor:l2rep}
\textbf{($\mathcal{L}_2$ Representation.)} For any SCM $\mathfrak{C}$ that induces a graph $\mathfrak{G}$, there exists a corresponding iVGAE $\mathcal{V}{=}(q,p)$ where $q,p$ are $\mathfrak{G}$-GNN layers such that $\mathcal{V}$ is $\mathcal{L}_2$-consistent with $\mathfrak{C}$.
\end{corollary}
The main claim of this theorem is two-fold, first, that the expressivity established in Thm.\ref{thm:express} is being preserved for $\mathfrak{G}$-GNN layers, and second, that the shared graph $\mathfrak{G}$ allows for tighter integration. 

\section{Identifiability and Estimation}
In line with \citep{xia2021causal}, we finally set stage for identifiability within the neural-causal iVGAE model by first extending their notion of neural identifiability, since iVGAE is a neural model that makes trades the set of feedforward networks for two interventional graph neural networks.
\begin{defin}\label{def:neuralident}
\textbf{(Neural Identifiability.)} Again, let $\Omega,\Upsilon$ be the sets of SCMs and corresponding $\mathfrak{G}$-GNN based iVGAE. For any pair $(\mathfrak{C},\mathcal{V})\in\Omega\times\Upsilon$, a causal effect $p(\mathbf{V}_i{\mid}\doop(\mathbf{V}_j))$ is called neurally identifiable iff the pair agrees on both the causal effect [$p^{\mathfrak{C}}(\mathbf{V}_i{\mid}\doop(\mathbf{V}_j))=p^{\mathcal{V}}(\mathbf{V}_i{\mid}\doop(\mathbf{V}_j))$] and the observational distributions [$p^{\mathfrak{C}}(\mathbf{V})=p^{\mathcal{V}}(\mathbf{V})$].
\end{defin}
This definition enforces a matching on all the relevant levels such that a causal effect becomes identifiable. With this, a final central claim on the relation of general identifiable effects and neurally identifiable effects (i.e., using the iVGAE as the model of choice) is being established.
\begin{theorem}\label{thm:dualid}
\textbf{(Dual Identification.)} Consider the causal quantity of interest $Q=p^{\mathfrak{C}}(\mathbf{V}_i{\mid}\doop(\mathbf{V}_j))$, $\mathfrak{G}$ the true causal graph underlying SCM $\mathfrak{C}\in\Omega$ and $p(\mathbf{V})$ the observational distribution. $Q$ is neurally identifiable from iVGAE $\mathcal{V}{\in}\Upsilon$ with $\mathfrak{G}$-GNN modules iff $Q$ is identifiable from $\mathfrak{G}$ and $p(\mathbf{V})$.
\end{theorem}
\begin{algorithm}[!t]
\caption{Causal Inference with GNN}
\label{alg:inference}
\textbf{Input}: iVGAE modules $\mathcal{V}_i$, Intervention $\mathbf{v}_j$, Variable $\mathbf{v}_i$\\
\textbf{Parameter}: Number of Samples $n$\\
\textbf{Output}: $\log(\hat{p}(\mathbf{v}_i{\mid}\doop(\mathbf{V}_j=\mathbf{v}_j)))$
\begin{algorithmic}[1] 
\STATE Let $\hat{p} = 0$ \\
\FOR{$i=1..n$}
\STATE Encoding $\mathbf{z}, \log(p(\mathbf{z}\doop(\mathbf{v}_j)))  \leftarrow \mathcal{V}_1(\mathbf{v}_i, \mathbf{v}_j)$
\STATE Decoding $\log(p(\mathbf{v}_i{\mid}, \mathbf{z}\doop(\mathbf{v}_j)))  \leftarrow \mathcal{V}_2(\mathbf{z}, \mathbf{v}_i, \mathbf{v}_j)$
\STATE \parbox[t]{\linewidth}{Aggregate $\hat{p}\leftarrow \exp(\log(p(\mathbf{v}_i{\mid}, \mathbf{z},\doop(\mathbf{v}_j))) - \log(p(\mathbf{z}{\mid}\doop(\mathbf{v}_j))))$}
\ENDFOR
\STATE \textbf{return} $\log(\hat{p}) - \log(n)$
\end{algorithmic}
\end{algorithm}
This theorem is analogues to the statement on neural identifiability in \citep{xia2021causal} thus following the same implications in that theoretically lower-to-higher-layer inference is possible within our neural setting, avoiding the usage of $\doop$-calculus altogether. The hope lies within the differentiability of the considered (neural) models to allow for leveraging causal inferences automatically from large-scale data. From a causal perspective, it is important to note that Def.\ref{def:neuralident} alongside the theoretical result Thm.\ref{thm:dualid} consider identifiability and not actual identification, the former refers to the power of cross-layer inferences while the latter refers to the process of using e.g.\ $\doop$-calculus for SCM or mutilation in the case of NCM to obtain the estimands\footnote{These estimands can then be estimated using data.}. For the iVGAE identification thus refers to cross-layer modelling which is a special case of multi-layer modelling, to which we will simply refer to as estimation. The estimation is performed using a modified version of the variational objective in Eq.\ref{eq:elbo} to respect the causal quantities, $ \mathbb{E}_q[\log p(\mathbf{V}{\mid} \mathbf{Z}, do(\mathbf{W}))] - \text{KL}(q(\mathbf{Z}{\mid} do(\mathbf{W})){\mid\mid}p(\mathbf{Z}))$, where $\mathbf{W}{\subset}\mathbf{V}$ are intervened variables and $p,q$ the $\mathfrak{G}$-GNNs of the iVGAE model. After optimizing the iVGAE model with this causal ELBO-variant, we can consider any quantity of interest dependent on the modelled levels $Q(\mathcal{L}_i)$. One interesting choice for $Q$ is the average treatment effect (ATE), $\text{ATE}(X,Y):=\mathbb{E}[Y{\mid}\doop(X{=}1)]-\mathbb{E}[Y{\mid}\doop(X{=}0)]$, where the binary\footnote{Note that we can extend the ATE to be categorical/continuous.} variable $X$ is being referred to as treatment. In Alg.\ref{alg:inference} we show how marginal inference under intervention is being performed.
\begin{figure}[!t]
\centering
\includegraphics[width=0.9\textwidth]{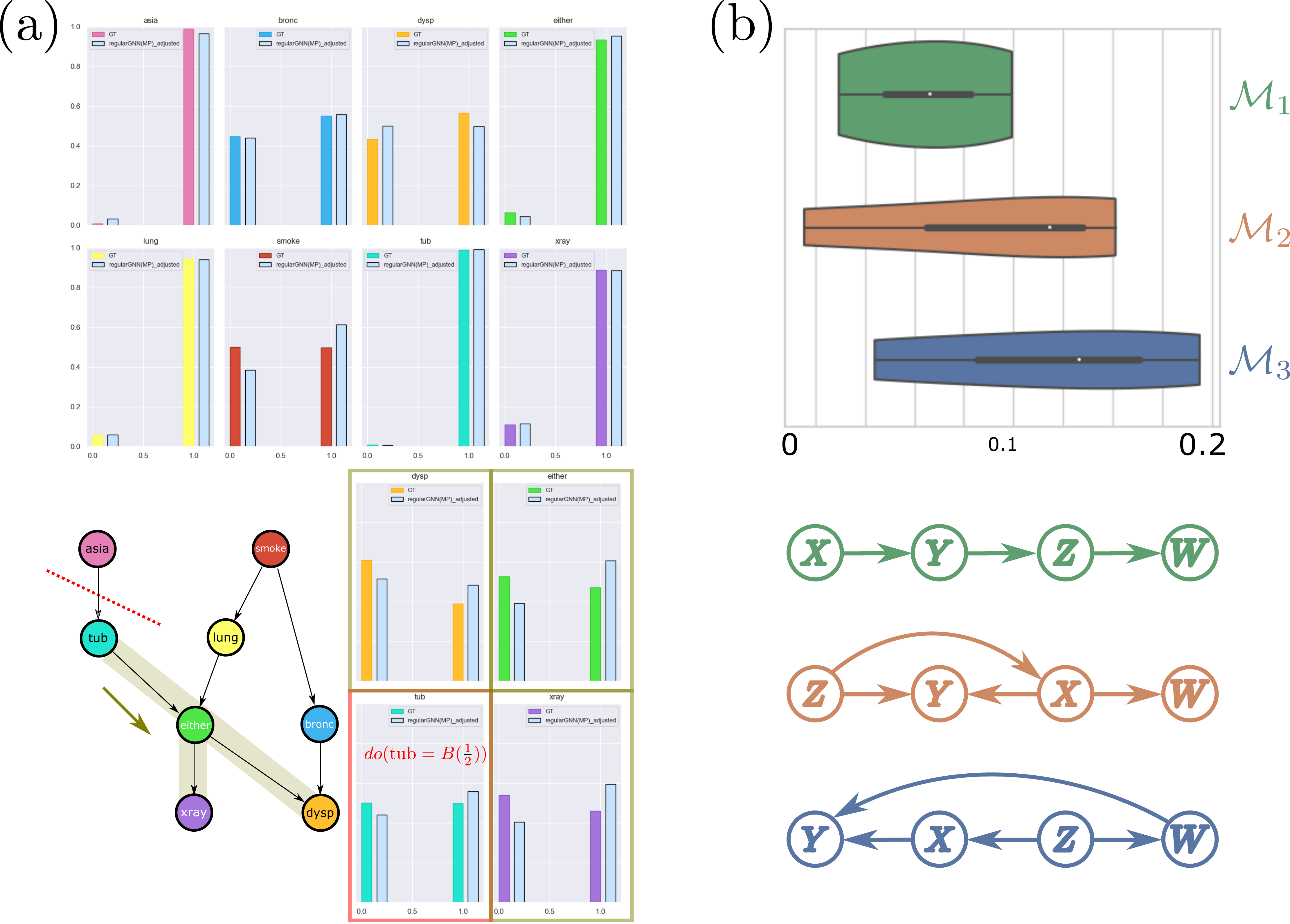}
\caption{\textbf{GNN Capture Causal Influence.} Empirical illustration (a) shows causal density estimation (top $\mathcal{L}1$, bottom $\mathcal{L}_2$ with $\doop(\text{tub}=\mathcal{B}(\frac{1}{2}))$) on ASIA where the property 'tub' turns into a fair coin flip (red box) that also causally influences its descendants (olive boxes) (b) shows estimation of ATE of $X$ on $Y$ for different SCM-families (Best viewed in color.)}
\label{fig:Density}
\vspace{-.75cm}
\end{figure}

\section{Empirical Illustration} \label{sec:emp}
\setlength{\parskip}{1em}
To assist our theoretical results we provide an extensive set of empirical illustrations. Due to space constraints, we only highlight relevant key results while pointing to the \textbf{extensive Appendix}.

\noindent{\bf Density Estimation.} We investigate properties of the causal density estimation when using iVGAE (Def.\ref{def:iVGAE}). We inspect various aspects like for instance how interventional influence propagates along the topological sequence or how precisely the model is able to fit. Fig.\ref{fig:Density}(a) shows an iVGAE model on the real-world inspired ASIA data set \citep{lauritzen1988local} that estimates the observational and interventional distributions from data sufficiently well in terms of both quality and recall. Inference is performed using Alg.\ref{alg:inference}. An extensive study is provided in the Appendix.

\noindent{\bf Causal Effect Estimation.} We investigate the estimation of the ATE of $X$ on $Y$ using iVGAE on different families of SCMs, that is, different structures with different parameterizations. Estimation considers the predictive quality amongst different parameterizations. In Fig.\ref{fig:Density}(b) chain ($\mathcal{M}_1$), confounder ($\mathcal{M}_2$), and backdoor ($\mathcal{M}_3$) structures are shown. The iVGAE model adequately captures the un- and confounded causal effects. An extensive elaboration is provided in the Appendix.
\setlength{\parskip}{0em}

\section{Conclusions and Future Work}
We derive the first theoretical connection between GNNs and SCMs (Thm.\ref{thm:gnnscm}) from first principles and in the process extend the family of neural causal models with two new classes: NCM-Type 2 (Cor.\ref{cor:ncmgnn}) and iVGAE (Def.\ref{def:iVGAE}). We provide several theoretical results on the feasibility, expressivity and identifiability (Thms.\ref{thm:express},\ref{thm:dualid}; Cors.\ref{cor:l3},\ref{cor:l2rep}). To support the theoretical results empirically, we systematically investigate practical causal inference on several benchmarks. Following, an extension that supports the complete PCH without feasibility trade-off would be desirable. Differently encoding the intervention within GNN (Def.\ref{def:iGNN}) or concepts like PMP \citep{strathmann2021persistent} might allow for this direction. Further, testing on larger causal systems is a pragmatic, natural next step.

\clearpage
\acks{This work was supported by the ICT-48 Network of AI Research Excellence Center “TAILOR" (EU Horizon 2020, GA No 952215) and by the Federal Ministry of Education and Research (BMBF; project “PlexPlain”, FKZ 01IS19081). It benefited from the Hessian research priority programme LOEWE within the project WhiteBox, the HMWK cluster project “The Third Wave of AI.” and the Collaboration Lab “AI in Construction” (AICO).}

\bibliography{clear2022}

\begin{thebibliography}{48}
\providecommand{\natexlab}[1]{#1}
\providecommand{\url}[1]{\texttt{#1}}
\expandafter\ifx\csname urlstyle\endcsname\relax
  \providecommand{\doi}[1]{doi: #1}\else
  \providecommand{\doi}{doi: \begingroup \urlstyle{rm}\Url}\fi

\bibitem[Bareinboim et~al.(2020)Bareinboim, Correa, Ibeling, and
  Icard]{bareinboim20201on}
Elias Bareinboim, Juan~D Correa, Duligur Ibeling, and Thomas Icard.
\newblock On pearl’s hierarchy and.
\newblock 2020.

\bibitem[Bica et~al.(2020)Bica, Alaa, and Van Der~Schaar]{bica2020time}
Ioana Bica, Ahmed Alaa, and Mihaela Van Der~Schaar.
\newblock Time series deconfounder: Estimating treatment effects over time in
  the presence of hidden confounders.
\newblock In \emph{ICML}, 2020.

\bibitem[Blei et~al.(2017)Blei, Kucukelbir, and McAuliffe]{blei2017variational}
David~M Blei, Alp Kucukelbir, and Jon~D McAuliffe.
\newblock Variational inference: A review for statisticians.
\newblock \emph{Journal of the American statistical Association}, 2017.

\bibitem[Bronstein et~al.(2017)Bronstein, Bruna, LeCun, Szlam, and
  Vandergheynst]{bronstein2017geometric}
Michael~M Bronstein, Joan Bruna, Yann LeCun, Arthur Szlam, and Pierre
  Vandergheynst.
\newblock Geometric deep learning: going beyond euclidean data.
\newblock \emph{IEEE Signal Processing Magazine}, 2017.

\bibitem[Bronstein et~al.(2021)Bronstein, Bruna, Cohen, and
  Veli{\v{c}}kovi{\'c}]{bronstein2021geometric}
Michael~M Bronstein, Joan Bruna, Taco Cohen, and Petar Veli{\v{c}}kovi{\'c}.
\newblock Geometric deep learning: Grids, groups, graphs, geodesics, and
  gauges.
\newblock \emph{arXiv preprint arXiv:2104.13478}, 2021.

\bibitem[Buchsbaum et~al.(2012)Buchsbaum, Bridgers, Skolnick~Weisberg, and
  Gopnik]{buchsbaum2012power}
Daphna Buchsbaum, Sophie Bridgers, Deena Skolnick~Weisberg, and Alison Gopnik.
\newblock The power of possibility: Causal learning, counterfactual reasoning,
  and pretend play.
\newblock \emph{Philosophical Transactions of the Royal Society B: Biological
  Sciences}, 2012.

\bibitem[Cybenko(1989)]{cybenko1989approximation}
George Cybenko.
\newblock Approximation by superpositions of a sigmoidal function.
\newblock \emph{Mathematics of control, signals and systems}, 1989.

\bibitem[Derrow-Pinion et~al.(2021)Derrow-Pinion, She, Wong, Lange, Hester,
  Perez, Nunkesser, Lee, Guo, Wiltshire, et~al.]{derrow2021eta}
Austin Derrow-Pinion, Jennifer She, David Wong, Oliver Lange, Todd Hester, Luis
  Perez, Marc Nunkesser, Seongjae Lee, Xueying Guo, Brett Wiltshire, et~al.
\newblock Eta prediction with graph neural networks in google maps.
\newblock \emph{CIKM}, 2021.

\bibitem[Gilmer et~al.(2017)Gilmer, Schoenholz, Riley, Vinyals, and
  Dahl]{gilmer2017neural}
Justin Gilmer, Samuel~S Schoenholz, Patrick~F Riley, Oriol Vinyals, and
  George~E Dahl.
\newblock Neural message passing for quantum chemistry.
\newblock In \emph{ICML}, 2017.

\bibitem[Gondal et~al.(2019)Gondal, W{\"u}thrich, Miladinovi{\'c}, Locatello,
  Breidt, Volchkov, Akpo, Bachem, Sch{\"o}lkopf, and Bauer]{gondal2019transfer}
Muhammad~Waleed Gondal, Manuel W{\"u}thrich, Dorde Miladinovi{\'c}, Francesco
  Locatello, Martin Breidt, Valentin Volchkov, Joel Akpo, Olivier Bachem,
  Bernhard Sch{\"o}lkopf, and Stefan Bauer.
\newblock On the transfer of inductive bias from simulation to the real world:
  a new disentanglement dataset.
\newblock \emph{arXiv preprint arXiv:1906.03292}, 2019.

\bibitem[Goodfellow et~al.(2016)Goodfellow, Bengio, and
  Courville]{goodfellow2016deep}
Ian Goodfellow, Yoshua Bengio, and Aaron Courville.
\newblock \emph{Deep learning}.
\newblock MIT press, 2016.

\bibitem[Goodfellow et~al.(2015)Goodfellow, Shlens, and
  Szegedy]{goodfellow2014explaining}
Ian~J Goodfellow, Jonathon Shlens, and Christian Szegedy.
\newblock Explaining and harnessing adversarial examples.
\newblock \emph{ICLR}, 2015.

\bibitem[Gopnik(2012)]{gopnik2012scientific}
Alison Gopnik.
\newblock Scientific thinking in young children: Theoretical advances,
  empirical research, and policy implications.
\newblock \emph{Science}, 2012.

\bibitem[Hair~Jr and Sarstedt(2021)]{hair2021data}
Joseph~F Hair~Jr and Marko Sarstedt.
\newblock Data, measurement, and causal inferences in machine learning:
  opportunities and challenges for marketing.
\newblock \emph{Journal of Marketing Theory and Practice}, 2021.

\bibitem[Hinton et~al.(2012)Hinton, Srivastava, and Swersky]{hinton2012neural}
Geoffrey Hinton, Nitish Srivastava, and Kevin Swersky.
\newblock Neural networks for machine learning lecture 6a overview of
  mini-batch gradient descent.
\newblock \emph{Cited on}, 2012.

\bibitem[Hoiles and Schaar(2016)]{hoiles2016bounded}
William Hoiles and Mihaela Schaar.
\newblock Bounded off-policy evaluation with missing data for course
  recommendation and curriculum design.
\newblock In \emph{ICML}, 2016.

\bibitem[Holland(1986)]{holland1986statistics}
Paul~W Holland.
\newblock Statistics and causal inference.
\newblock \emph{Journal of the American statistical Association}, 1986.

\bibitem[Hornik(1991)]{hornik1991approximation}
Kurt Hornik.
\newblock Approximation capabilities of multilayer feedforward networks.
\newblock \emph{Neural networks}, 1991.

\bibitem[Hornik et~al.(1989)Hornik, Stinchcombe, and
  White]{hornik1989multilayer}
Kurt Hornik, Maxwell Stinchcombe, and Halbert White.
\newblock Multilayer feedforward networks are universal approximators.
\newblock \emph{Neural networks}, 1989.

\bibitem[Jordan et~al.(1999)Jordan, Ghahramani, Jaakkola, and
  Saul]{jordan1999introduction}
Michael~I Jordan, Zoubin Ghahramani, Tommi~S Jaakkola, and Lawrence~K Saul.
\newblock An introduction to variational methods for graphical models.
\newblock \emph{Machine learning}, 1999.

\bibitem[Kingma and Welling(2019)]{kingma2019introduction}
Diederik~P Kingma and Max Welling.
\newblock An introduction to variational autoencoders.
\newblock \emph{arXiv preprint arXiv:1906.02691}, 2019.

\bibitem[Kipf et~al.(2018)Kipf, Fetaya, Wang, Welling, and
  Zemel]{kipf2018neural}
Thomas Kipf, Ethan Fetaya, Kuan-Chieh Wang, Max Welling, and Richard Zemel.
\newblock Neural relational inference for interacting systems.
\newblock In \emph{ICML}, 2018.

\bibitem[Kipf and Welling(2016{\natexlab{a}})]{kipf2016semi}
Thomas~N Kipf and Max Welling.
\newblock Semi-supervised classification with graph convolutional networks.
\newblock \emph{arXiv preprint arXiv:1609.02907}, 2016{\natexlab{a}}.

\bibitem[Kipf and Welling(2016{\natexlab{b}})]{kipf2016variational}
Thomas~N Kipf and Max Welling.
\newblock Variational graph auto-encoders.
\newblock \emph{arXiv preprint arXiv:1611.07308}, 2016{\natexlab{b}}.

\bibitem[Koller and Friedman(2009)]{koller2009probabilistic}
Daphne Koller and Nir Friedman.
\newblock \emph{Probabilistic graphical models: principles and techniques}.
\newblock MIT press, 2009.

\bibitem[Korb and Nicholson(2010)]{korb2010bayesian}
Kevin~B Korb and Ann~E Nicholson.
\newblock \emph{Bayesian artificial intelligence}.
\newblock CRC press, 2010.

\bibitem[Krizhevsky et~al.(2012)Krizhevsky, Sutskever, and
  Hinton]{krizhevsky2012imagenet}
Alex Krizhevsky, Ilya Sutskever, and Geoffrey~E Hinton.
\newblock Imagenet classification with deep convolutional neural networks.
\newblock \emph{NeurIPS}, 2012.

\bibitem[Kuo et~al.(2010)Kuo, Sloan, Wasilkowski, and
  Wo{\'z}niakowski]{kuo2010decompositions}
F~Kuo, I~Sloan, Grzegorz Wasilkowski, and Henryk Wo{\'z}niakowski.
\newblock On decompositions of multivariate functions.
\newblock \emph{Mathematics of computation}, 2010.

\bibitem[Lauritzen and Spiegelhalter(1988)]{lauritzen1988local}
Steffen~L Lauritzen and David~J Spiegelhalter.
\newblock Local computations with probabilities on graphical structures and
  their application to expert systems.
\newblock \emph{Journal of the Royal Statistical Society: Series B
  (Methodological)}, 1988.

\bibitem[McCarthy(1998)]{mccarthy1998artificial}
John McCarthy.
\newblock What is artificial intelligence?
\newblock 1998.

\bibitem[McCarthy and Hayes(1981)]{mccarthy1981some}
John McCarthy and Patrick~J Hayes.
\newblock Some philosophical problems from the standpoint of artificial
  intelligence.
\newblock In \emph{Readings in artificial intelligence}. 1981.

\bibitem[Mitrovic et~al.(2020)Mitrovic, McWilliams, Walker, Buesing, and
  Blundell]{mitrovic2020representation}
Jovana Mitrovic, Brian McWilliams, Jacob Walker, Lars Buesing, and Charles
  Blundell.
\newblock Representation learning via invariant causal mechanisms.
\newblock \emph{arXiv preprint arXiv:2010.07922}, 2020.

\bibitem[Mnih et~al.(2013)Mnih, Kavukcuoglu, Silver, Graves, Antonoglou,
  Wierstra, and Riedmiller]{mnih2013playing}
Volodymyr Mnih, Koray Kavukcuoglu, David Silver, Alex Graves, Ioannis
  Antonoglou, Daan Wierstra, and Martin Riedmiller.
\newblock Playing atari with deep reinforcement learning.
\newblock \emph{arXiv preprint arXiv:1312.5602}, 2013.

\bibitem[Pawlowski et~al.(2020)Pawlowski, Castro, and
  Glocker]{pawlowski2020deep}
Nick Pawlowski, Daniel~C Castro, and Ben Glocker.
\newblock Deep structural causal models for tractable counterfactual inference.
\newblock \emph{arXiv preprint arXiv:2006.06485}, 2020.

\bibitem[Pearl(2009)]{pearl2009causality}
Judea Pearl.
\newblock \emph{Causality}.
\newblock Cambridge university press, 2009.

\bibitem[Pearl and Mackenzie(2018)]{pearl2018book}
Judea Pearl and Dana Mackenzie.
\newblock \emph{The book of why: the new science of cause and effect}.
\newblock Basic books, 2018.

\bibitem[Pearl et~al.(2016)Pearl, Glymour, and Jewell]{pearl2016causal}
Judea Pearl, Madelyn Glymour, and Nicholas~P Jewell.
\newblock \emph{Causal inference in statistics: A primer}.
\newblock John Wiley \& Sons, 2016.

\bibitem[Penn and Povinelli(2007)]{penn2007causal}
Derek~C Penn and Daniel~J Povinelli.
\newblock Causal cognition in human and nonhuman animals: A comparative,
  critical review.
\newblock \emph{Annu. Rev. Psychol.}, 2007.

\bibitem[Peters et~al.(2017)Peters, Janzing, and
  Sch{\"o}lkopf]{peters2017elements}
Jonas Peters, Dominik Janzing, and Bernhard Sch{\"o}lkopf.
\newblock \emph{Elements of causal inference}.
\newblock The MIT Press, 2017.

\bibitem[Plataniotis and Hatzinakos(2017)]{plataniotis2017gaussian}
Kostantinos~N Plataniotis and Dimitris Hatzinakos.
\newblock Gaussian mixtures and their applications to signal processing.
\newblock In \emph{Advanced signal processing handbook}. 2017.

\bibitem[Rubinstein and Kroese(2016)]{rubinstein2016simulation}
Reuven~Y Rubinstein and Dirk~P Kroese.
\newblock \emph{Simulation and the Monte Carlo method}.
\newblock John Wiley \& Sons, 2016.

\bibitem[Steinruecken et~al.(2019)Steinruecken, Smith, Janz, Lloyd, and
  Ghahramani]{steinruecken2019automatic}
Christian Steinruecken, Emma Smith, David Janz, James Lloyd, and Zoubin
  Ghahramani.
\newblock The automatic statistician.
\newblock In \emph{Automated Machine Learning}. 2019.

\bibitem[Stokes et~al.(2020)Stokes, Yang, Swanson, Jin, Cubillos-Ruiz, Donghia,
  MacNair, French, Carfrae, Bloom-Ackermann, et~al.]{stokes2020deep}
Jonathan~M Stokes, Kevin Yang, Kyle Swanson, Wengong Jin, Andres Cubillos-Ruiz,
  Nina~M Donghia, Craig~R MacNair, Shawn French, Lindsey~A Carfrae, Zohar
  Bloom-Ackermann, et~al.
\newblock A deep learning approach to antibiotic discovery.
\newblock \emph{Cell}, 2020.

\bibitem[Strathmann et~al.(2021)Strathmann, Barekatain, Blundell, and
  Veli{\v{c}}kovi{\'c}]{strathmann2021persistent}
Heiko Strathmann, Mohammadamin Barekatain, Charles Blundell, and Petar
  Veli{\v{c}}kovi{\'c}.
\newblock Persistent message passing.
\newblock \emph{arXiv preprint arXiv:2103.01043}, 2021.

\bibitem[Vaswani et~al.(2017)Vaswani, Shazeer, Parmar, Uszkoreit, Jones, Gomez,
  Kaiser, and Polosukhin]{vaswani2017attention}
Ashish Vaswani, Noam Shazeer, Niki Parmar, Jakob Uszkoreit, Llion Jones,
  Aidan~N. Gomez, Lukasz Kaiser, and Illia Polosukhin.
\newblock Attention is all you need.
\newblock In \emph{NeurIPS}, 2017.

\bibitem[Veli{\v{c}}kovi{\'c} et~al.(2017)Veli{\v{c}}kovi{\'c}, Cucurull,
  Casanova, Romero, Lio, and Bengio]{velivckovic2017graph}
Petar Veli{\v{c}}kovi{\'c}, Guillem Cucurull, Arantxa Casanova, Adriana Romero,
  Pietro Lio, and Yoshua Bengio.
\newblock Graph attention networks.
\newblock \emph{arXiv preprint arXiv:1710.10903}, 2017.

\bibitem[Wiener(1932)]{wiener1932tauberian}
Norbert Wiener.
\newblock Tauberian theorems.
\newblock \emph{Annals of mathematics}, 1932.

\bibitem[Xia et~al.(2021)Xia, Lee, Bengio, and Bareinboim]{xia2021causal}
Kevin Xia, Kai-Zhan Lee, Yoshua Bengio, and Elias Bareinboim.
\newblock The causal-neural connection: Expressiveness, learnability, and
  inference.
\newblock 2021.

\end{thebibliography}

\appendix

\clearpage 
\paragraph{Appendix for Relating Graph Neural Networks to Structural Causal Models.} We make use of this appendix following the main paper to provide the proofs to the main theorems, propositions and corollaries in addition to systematic investigations regarding practical causal inference in terms of causal effect estimation and general density estimation.

\setcounter{theorem}{0}
\setcounter{prop}{0}
\setcounter{corollary}{0}

\section{Proofs}
This section provides all the proofs for the mathematical results established in the main paper.

\subsection{Proofs for Theorem \ref{thm:gnnscm} and Corollary \ref{cor:ncmgnn}}
While Thm.\ref{thm:gnnscm} does not talk about optimization and feasibility, still it suggests that indeed GNN can be converted into SCM by suggesting that there always exists at least one construction of such. In the following we prove the theorem by giving a general construction scheme.
\begin{theorem}\label{thm:gnnscm}
\textbf{(GNN-SCM Conversion.)} Consider the most general formulation of a message-passing GNN node computation $\mathbf{h}_i{:} \mathcal{F}\mapsto \mathcal{F}^{\prime}$ as in Eq.\ref{eq:gnn}. For any SCM $\mathfrak{C}{=}(\mathbf{S},P(\mathbf{U}))$ there exists always a choice of feature spaces $\mathcal{F},\mathcal{F}^{\prime}$ and shared functions $\phi,\psi$, such that for all structural equations $f{\in}\mathbf{S}$ it holds that $\mathbf{h}_i = f_i$.
\end{theorem}
\begin{proof}
It is sufficient to provide a general construction scheme on SCMs. Therefore, let $\mathfrak{C}{=}(\mathbf{S},P(\mathbf{U}))$ be any SCM. Let $f_i(\pa(i),U_i)=f_{i}(U_i, \mathcal{A}_i)+\sum_{j\in\pa(i)} f_{ij}(V_j)$ be a structural equation ($f_i\in\mathbf{S}$) and its scalar-decomposition following Thm.1 in \citep{kuo2010decompositions} where $\mathcal{A}{\in}2^{|\pa_i|}$ is a potentially empty argument list. We choose the following mapping for the respective GNN computation components:
\begin{align}
\mathcal{F} &= \mathbf{V}\cup\mathbf{U} = \mathcal{F}^{\prime} \\
\phi(i,\dots) &= f_{i}(U_i, \mathcal{A}_i) + \sum \dots \\
\psi(i,j) &= f_{ij}
\end{align} where ($\cdots$) represents the remainder of the GNN-computation. Finally, it holds that
\begin{align}
\begin{split}
\mathbf{h}_i &= \phi\bigg(\mathbf{d}_i, \bigoplus_{j\in\mathcal{N}^G_i} \psi(\mathbf{d}_i, \mathbf{d}_j)\bigg) \\
&= f_{i}(U_i, \mathcal{A}_i) + \sum_{j\in\pa(i)} f_{ij}(V_j) = f_i.
\end{split}
\end{align}
\end{proof} A simple deduction of Thm.\ref{thm:gnnscm} in which we allow for the violation of sharedness, which lies at the core of the GNN formulation, leads to the formulation of a more fine-grained NCM model than what has been defined by \cite{xia2021causal}. It is more fine-grained in that this NCM-Type 2 operates on the edge-level opposed to the node-level as for regular NCM.
\begin{corollary}\label{cor:ncmgnn}
\textbf{(NCM-Type 2.)} Allowing for the violation of sharedness of $\psi$ as dicated in Thm.\ref{thm:gnnscm} and choosing $\mathcal{F}=\mathcal{F}^{\prime}=\mathbf{U}\cup\mathbf{V}$ to be the union over endo- and exogenous variables, $\phi(i,\dots)=f_{i}(U_i, \mathcal{A}_i)+\sum(\dots)$ to be a sum-aggregator with noise term selection with $\mathcal{A}_i{\in} 2^{|\pa_i|}$, and $\psi=\{f^{ij}_{\pmb{\theta}}\}$ to be the dependency terms of the structural equations $f_i$ modelled as feedforward neural networks. Then the computation layer $\{\mathbf{h}_i\}^{|V|}_i$ is a special case of the NCM as in \citep{xia2021causal}.
\end{corollary}
\begin{proof}
The proposed computation layer (NCM-Type 2) is a special case in the way it specifies the function approximators and in that it covers non-Markovianity i.e., no hidden confounding or relations on the noise terms. A NCM $\mathfrak{N}$ is specified the same way as an SCM $\mathfrak{C}$ with the difference being that the the noise terms being uniformly distributed over the intervall $[0,1]$, that is $U\sim\text{Unif}(0,1)$, and the structural equations being parameterized by feedforward neural networks $f_i := \text{MLP}_i^{\pmb{\theta}}$ with learnable parameters $\pmb{\theta}$. \citep{kuo2010decompositions} suggests we know that $f_i = f_{i}(U_i, \mathcal{A}_i) + \sum_{j\in\pa(i)} f_{ij}(V_j)$. Furthermore, we know that $\text{MLP}_i^{\pmb{\theta}} = \sum_k \text{MLP}_k^{\pmb{\theta}_k}$ \citep{hornik1989multilayer}. Thereby, we have that 
\begin{align}
\begin{split}
\mathbf{h}_i = f_{i}(U_i, \mathcal{A}_i) + \sum_j f_{ij}(V_j) = \sum_j \text{MLP}_j^{\pmb{\theta}_j} = \text{MLP}_i^{\pmb{\theta}} = f_i
\end{split}
\end{align} where $i\in\{1\dots |V|\}, j\in\pa(i)$ and $\mathcal{A}_i\in2^{|\pa_i|}$.
\end{proof}

\subsection{Proof for Proposition \ref{prop:equiv}}
Prop.\ref{prop:equiv} reassures that the established connection between SCM and GNN based on the graph that is being induced by the former is natural. More specifically, an intervention within such a $\mathfrak{G}$-GNN will produce the same mutilated graph as an intervention within the SCM.
\begin{prop}\label{prop:equiv}
\textbf{(Graph Mutilation Equivalence.)} Let $\mathfrak{C}$ be an SCM with graph $\mathfrak{G}$ and let $f$ be a $\mathfrak{G}$-GNN layer. An intervention $\doop(\mathbf{X}), \mathbf{X}\subseteq V,$ on both $\mathfrak{C}$ and $f$ produces the same mutilated graph.
\end{prop}
\begin{proof}
Following Def.3 in \citep{bareinboim20201on}, an interventional SCM $\mathfrak{C}_{\mathbf{x}}$ is a submodel of the original SCM $\mathfrak{C}$ where the structural equations for variables $\mathbf{X}$ are being replaced by the assignment $\mathbf{x}$. Through this operation, denoted by $\doop(\mathbf{X}=\mathbf{x})$, the dependency between the causal parents of any node $V_i\in\mathbf{X}$ is being lifted (as long as the assignment $\mathbf{x}$ is not dependent on the parents). Therefore the mutilated graph is given by $\mathfrak{G}_{\mathbf{x}}=(\mathbf{V}, \mathfrak{G}^E\setminus\{(j,i)\mid V_j\in\pa_i, V_i\in\mathbf{X}\}$ where $\mathfrak{G}^E$ denotes the edge list of the original graph. Intervening on a $\mathfrak{G}$-GNN layer implicitly considers a modified neighborhood $\mathcal{M}^G_i = \{j \mid j \in \mathcal{N}^G_i, j\not\in \pa_i {\iff} i\in\mathbf{X} \}$ which removes exactly the relatons to the parents. Since $G=\mathfrak{G}$, the mutilated graphs are the same.
\end{proof}

\subsection{Proof for Theorem \ref{thm:uda}}
Thm.\ref{thm:uda} suggests a long standing result on universal approximators for densities \citep{goodfellow2016deep,plataniotis2017gaussian} but applied to the data-driven VGAE. Thus the proof follows analogously.
\begin{theorem}\label{thm:uda}
\textbf{(Universal Density Approximation.)} There exists a latent variable family $q$ and a corresponding data-generative VGAE $\mathcal{V}{=}(q,p)$ such that the modelled density can approximate any smooth density at any nonzero error.
\end{theorem}
\begin{proof}
It is sufficient for the proof to show one example of a suitable family, that is encodable by a data-driven VGAE (Def.\ref{def:dVGAE}) and that can act as a universal approximator of densities. Let's inspect the Gaussian Mixture Model family of latent variable models. We will use Wiener's approximation theorem \citep{wiener1932tauberian} and knowledge delta functions of positive type. A delta family of positive type is defined under following conditions
\begin{enumerate}
    \item $\int_{-a}^a \delta_{\lambda}(x) \,dx \rightarrow\lambda$ as $\lambda\rightarrow\lambda_0$ for some $a$.
    \item For every constant $\gamma > 0$, $\delta_{\lambda}$ tends to zero uniformly for $\gamma\leq|x|\leq\infty$ as $\lambda\rightarrow\lambda_0$.
    \item $\delta_{\lambda}(x)\geq 0$ for all $x$ and $\lambda$.
\end{enumerate} and if it additionally satisfies $\int_{-\infty}^{\infty} \delta_{\lambda}(x)\,dx=1$, then such a function definies a probability density for all $\lambda$. A Gaussian density is delta in the zero limit of its variance. The sequence $p_{\lambda}(x)$ formed by the convolution of the delta function $\delta_{\lambda}$ and an arbitrary density function $p$ can be expressed in terms of Gaussian. Thus,
\begin{align}
\begin{split}
    p_{\lambda}(x) &= \int_{-\infty}^{\infty} \delta_{\lambda}(x-u)p(u)\,du \\ &= \int_{-\infty}^{\infty} \mathcal{N}_{\lambda}(x-u)p(u)\,du,
\end{split}
\end{align} forms the basis for the sum of Gaussians. It follows that $p_{\lambda}(x)$ can in fact be approximated with a Riemann sum $p_{\lambda}(x) = \frac{1}{k}\sum_{i=1}^n \mathcal{N}_{\lambda}(x-x_i)[\xi_i-\xi_{i-1}]$ on some interval $(a,b)$ with interval-boundary points $\xi_i,\xi_0=a,\xi_n=b$. It follows that we can express $p_{\lambda,n}(x)$ as a convex combination of different variance Gaussians. Finally, the sought density can be expressed as
\begin{equation}
    p(x) = \sum_{i=1}^k w_i \mathcal{N}(x; \mu_i,\Sigma)
\end{equation} with $\sum w = 1$ and $w_i\geq 0,\forall i$.
Note that $x=\mathbf{x}\in\mathbb{R}^d$ holds for all the previously established results. Thus a GMM can approximate any density. To conclude, a GMM can be modelled by data-driven VGAE by designing the output of the encoder $f_{\pmb{\theta}}(\mathbf{D},\mathbf{A})=\mathbf{Z}$ such that $\mathbf{z}_i:=w_i,\mu_i,\Sigma_i$.
\end{proof}

\subsection{Proofs for Theorem \ref{thm:express} and Corollaries \ref{cor:cht} and \ref{cor:l3}}
The iVGAE is a feasible model opposed to the GNN-reparameterization from Thm.\ref{thm:gnnscm} for the price of expressivity. Nonetheless, the following Thm.\ref{thm:express} reassures that the iVGAE can model causal quantites up to the second level of the PCH, namely that of interventions ($\mathcal{L}_2$).
\begin{theorem}\label{thm:express}
\textbf{(Expressivity.)} For any SCM $\mathfrak{C}$ there exists an iVGAE $\mathcal{V}(\pmb{\theta},\pmb{\phi})$ for which $\mathcal{V}$ is $\mathcal{L}_2$-consistent w.r.t\ $\mathfrak{C}$.
\end{theorem}
\begin{proof}
Let $\mathcal{V}(\pmb{\theta},\pmb{\phi})$ be an iVGAE and $\mathbf{D}=\{\mathbf{d}_i\}_{i=1}^{|\mathbf{V}|}$ a data set on the variables $\mathbf{V}$ of an arbitrary SCM $\mathfrak{C}$ for multiple interventions of said SCM, that is, $\mathbf{d}_i\sim p_k\in\mathcal{L}_j(\mathfrak{C})$ for $j\in\{1,2\}$ and $k>1$. Note that the observational case ($\mathcal{L}_1$) is considered to be an intervention on the empty set, $p(\mathbf{V}{\mid}\doop(\emptyset))=p(\mathbf{V})$. Through Thm.\ref{thm:uda} we know that there always exists a parameterization $\theta,\phi$ for $\mathcal{V}$ such that any distribution $p$ can be modelled to an arbitrary precision, thus $p^{\mathcal{V}}=p_k$. Since $k>1$ we have that $\mathcal{V}$ models the PCH up to level partially $\mathcal{L}_2(\mathcal{V})$. Finally, the distributions are modelled relative to $\mathfrak{C}$ and $\mathcal{L}_2(\mathcal{V})\subset \mathcal{L}_2(\mathfrak{C})$.
\end{proof}
It is important to note that the $\mathcal{L}_i$-consistency defined in Def.\ref{def:consistency} is a weakened notion of consistency across causal models since any model might only agree on a single distribution of a given level $p\in\mathcal{L}_i$. Due to the nature of iVGAE being a compact model class opposed to a neural copy of an SCM, this is the only consistency achievable. However, it is not a negative impossibility result, on the contrary, iVGAE can perform any causal inference within the first two rungs of the PCH when provided with the corresponding data and sufficient model capacity i.e., the trade-off in expressivity solely comes from the fact that a compression on the model description is being performed. An important corollary to Thm.\ref{thm:express} is that the CHT still applies across settings. That is, causal inferences within iVGAE as choice of parameterization remain sensible since the layer boundaries are strict.
\begin{corollary}\label{cor:cht}
\textbf{(iVGAE Partial Causal Hierarchy Theorem.)} Consider the sets of all SCM and iVGAE, $\Omega,\Upsilon$, respectively. If for all $\mathcal{V}\in\Upsilon$ it holds that $\mathcal{V}$ is $\mathcal{L}^p_1(\mathcal{V}){=}\mathcal{L}^p_1(\mathfrak{C})\implies\mathcal{L}^q_2(\mathcal{V}){=}\mathcal{L}^q_2(\mathfrak{C})$ with $\mathfrak{C}\in\Omega$, where $\mathcal{L}^p_i \subset \mathcal{L}_i$ is a selection $p,q$ over the set of all level 1,2 distributions respectively, then we say that layer 2 of iVGAE collapses relative to $\mathfrak{C}$. On the Lebesgue measure over SCMs, the subset in which layer 2 of iVGAE collapses to layer 1 has measure zero. 
\end{corollary}
\begin{proof}
The proof is analogue to \citep{xia2021causal,bareinboim20201on} and assumes Fact 1 (or Thm.1 for the latter) to define an SCM-collapse relative to some SCM $\mathfrak{C}^{\prime}$. If layer 2 SCM-collapses to layer 1 relative to $\mathfrak{C}^{\prime}$ then any SCM $\mathfrak{C}$ will follow the properties $\mathcal{L}_1^p(\mathfrak{C})=\mathcal{L}_1^p(\mathfrak{C}^{\prime})$ and $\mathcal{L}_2^k(\mathfrak{C})=\mathcal{L}_2^k(\mathfrak{C}^{\prime})$ for some set selections $p,k$. By Thm.\ref{thm:express} we know that there will always exist a corresponding iVGAE $\mathcal{V}$ that is $\mathcal{L}_2$-consistent with $\mathfrak{C}$ but since it is consistent with $\mathfrak{C}$ and not $\mathfrak{C}^{\prime}$ it follows the equivalent properties $\mathcal{L}_1^p(\mathcal{V})=\mathcal{L}_1^p(\mathfrak{C}^{\prime})$ and $\mathcal{L}_2^k(\mathcal{V})=\mathcal{L}_2^k(\mathfrak{C}^{\prime})$, which means that the layer 2 also collapses for the iVGAE model. The analogue argument holds in reverse when layer 2 does not SCM-collapse to layer 1 relative to $\mathfrak{C}^{\prime}$. Since both directions together suggest an equivalence on the way collapse occurs for both SCM and iVGAE, Fact 1 from \citep{xia2021causal} (or Thm.1 from \citep{bareinboim20201on}) holds.
\end{proof}
The iVGAE model is capable of causal inference but it cannot act as a complete reparamterization of an SCM since it lacks the same amount of model description, that is, an SCM models structural equations for each of the variables while iVGAE expresses a causal probabilistic model paramterized by two functions. A consequence of Thm.\ref{thm:express} is thus that the general layer of counterfactuals ($\mathcal{L}_3$) cannot be inferred using the description in Def.\ref{def:iVGAE}. While it is easy to see why this is the case, since iVGAE deploys a single model, it is important to consider Def.2 from \citep{bareinboim20201on} to proceed with the proof.
\begin{defin}\label{def:eliasl3val}
\textbf{($\mathcal{L}_3$ Valuations.)}
Let $\mathfrak{C}$ be an SCM, then for instantiations $\mathbf{x},\mathbf{y}$ of the node sets $\mathbf{X},\mathbf{Y},\mathbf{Z},\mathbf{W}\dots \subseteq \mathbf{V}$ where $\mathbf{Y}_{\mathbf{x}}: \mathbf{U}\mapsto \mathbf{Y}$ denotes the value of $\mathbf{Y}$ under intervention ${\mathbf{x}}$, a counterfactual distribution is given by
\begin{equation}
    p^{\mathfrak{C}}(\mathbf{y}_{\mathbf{x}},\mathbf{z}_{\mathbf{w}}, \dots) = \sum_{u\in\mathcal{U}} p(\mathbf{u})
\end{equation} with $\mathcal{U}=\{\mathbf{u}\mid \mathbf{Y}_{\mathbf{x}}(\mathbf{u})=\mathbf{y}, \mathbf{Z}_{\mathbf{w}}(\mathbf{u})=\mathbf{z}, \dots \}$ being all noise instantiations consistent with $\mathbf{Y}_{\mathbf{x}},\mathbf{Z}_{\mathbf{w}},\dots$ representing different "worlds".
\end{defin}
With Def.\ref{def:eliasl3val} the simplicity of the proof for Cor.\ref{cor:l3} becomes evident, since a single iVGAE is not capable of representing the different (counter-factual) "worlds" implied by the different interventions $p(\mathbf{V}_i{\mid}\doop(\mathbf{V}_j)) = \mathbf{V}_{i,\mathbf{v}_j}$.
\begin{corollary}\label{cor:l3}
\textbf{(iVGAE Limitation.)} For any SCM $\mathfrak{C}$ there exists no iVGAE $\mathcal{V}$ such that $\mathcal{V}$ is $\mathcal{L}_3$-consistent w.r.t.\ $\mathfrak{C}$.
\end{corollary}
\begin{proof}
For an iVGAE model $\mathcal{V}$ to be $\mathcal{L}_3$-consistent with an SCM requires $\mathcal{V}$ to have the capability of inducing distributions from that layer $p^{\mathcal{V}}=p\sim\mathcal{L}_3$. Def.\ref{def:eliasl3val} suggests that a counterfactual involves the instantiations of multiple worlds for any SCM $\mathfrak{C}$, $p(\mathbf{V}_i{\mid}\doop(\mathbf{V}_j))$ for different tuples $(i,j)$. While $\mathcal{V}$ is capable of modelling each of the $\omega_{ij}=p(\mathbf{V}_i{\mid}\doop(\mathbf{V}_j))$ jointly, it is not capable of accumulating the probability masses $p(\mathbf{u})$ for which $\mathbf{u}$ is consistent with $\cup \omega_{ij}$ since it is a single model that directly estimates the l.h.s.\ of Eq.\ref{eq:l12val} opposed to an algorithmic procedure of checking over all possible instances $\mathbf{u}$ for consistency.
\end{proof}

\subsection{Proof for Corollary \ref{cor:l2rep}}
While Thm.\ref{thm:express} points out that iVGAE are causally expressive and that their existence is theoretically guaranteed, the following corollary is an important consequence that reduces the search space for such an iVGAE in practice significantly. That is, the induced graph $\mathfrak{G}$ of any SCM $\mathfrak{C}$ can be inducted into the GNN layers of the iVGAE such that $\mathcal{L}_2$-consistency is still warranted.
\begin{corollary}\label{cor:l2rep}
\textbf{($\mathcal{L}_2$ Representation.)} For any SCM $\mathfrak{C}$ that induces a graph $\mathfrak{G}$, there exists a corresponding iVGAE $\mathcal{V}{=}(q,p)$ where $q,p$ are $\mathfrak{G}$-GNN layers such that $\mathcal{V}$ is $\mathcal{L}_2$-consistent with $\mathfrak{C}$.
\end{corollary}
\begin{proof}
We repeat the arguments of the proof in Thm.\ref{thm:express} with the slight modification that $\mathcal{V}:=\mathcal{V}^{\mathfrak{G}}$, indicating the $\mathfrak{G}$-GNN layers $p,q$, where $\mathfrak{G}$ is the SCM $\mathfrak{C}$ induced graph.
\end{proof}

\subsection{Proof for Theorem \ref{thm:dualid}}
While expressivity is an important property of any model for inferential processes in science, identifiability stands at the core of causal inference. That is, using partial information on higher levels of the PCH (or of the SCM) to perform inferences starting from the lower levels. E.g., inferring the causal effect of $X$ on $Y$ denoted by $p(Y\mid \doop(X))$ using only observational data from $p(X,Y)$ requires the identification of an estimand in the purely causal setting. If the graph is given by $X\rightarrow Y$, then the identification is given by the conditional $p(Y\mid X) = p(Y\mid \doop(X))$. However, if the structure is more involved, imagine a (hidden) confounder $Z$ such that $X\leftarrow Z\rightarrow Y$, then an adjustment regarding $Z$'s influence would be interesting. Pearl's $\doop$-calculus provides an graphical-algebraic tool for performing identification \citep{pearl2009causality}. Thm.\ref{thm:dualid} establishes that these causal inferences are equivalent between the domains. In a nutshell, if one can perform an identification of inference for an SCM, one can also do it for an iVGAE model.
\begin{theorem}\label{thm:dualid}
\textbf{(Dual Identification.)} Consider the causal quantity of interest $Q=p^{\mathfrak{C}}(\mathbf{V}_i{\mid}\doop(\mathbf{V}_j))$, $\mathfrak{G}$ the true causal graph underlying SCM $\mathfrak{C}\in\Omega$ and $p(\mathbf{V})$ the observational distribution. $Q$ is neurally identifiable from iVGAE $\mathcal{V}{\in}\Upsilon$ with $\mathfrak{G}$-GNN modules iff $Q$ is identifiable from $\mathfrak{G}$ and $p(\mathbf{V})$.
\end{theorem}
\begin{proof}
The proof uses the same trick as the proof for Thm.4 in \citep{xia2021causal} to establish the duality in identification for SCM and iVGAE. Let $\Omega,\Upsilon$ be the sets of all SCMs and iVGAEs respectively. If $Q$ is not identifiable from a graph $\mathfrak{G}$ and the observational distribution $p(\mathbf{V})$ with full support $\forall \mathbf{v}: p(\mathbf{v})>0$, then there exists a pair of SCMs $\mathfrak{C},\mathfrak{C}^{\prime}$ such that they agree on $\mathcal{L}_1$ and $\mathfrak{G}$ but not on $\mathcal{L}_2$ i.e., $p^{\mathfrak{C}}(\mathbf{V}_i{\mid}\doop(\mathbf{V}_j))\neq p^{\mathfrak{C}^{\prime}}(\mathbf{V}_i{\mid}\doop(\mathbf{V}_j))$. By Cor.\ref{cor:l2rep} we know that there will always be an iVGAE model $\mathcal{V}^{\mathfrak{G}}$ based on $\mathfrak{G}$-GNN layers that is $\mathcal{L}_2$-consistent with $\mathfrak{C}$. This is all, since, $p^{\mathcal{V}^{\mathfrak{G}}}(\mathbf{V}_i{\mid}\doop(\mathbf{V}_j))= p^{\mathfrak{C}}(\mathbf{V}_i{\mid}\doop(\mathbf{V}_j))\neq p^{\mathfrak{C}^{\prime}}(\mathbf{V}_i{\mid}\doop(\mathbf{V}_j))$ which suggests that if $Q$ is not identifiable by $\mathfrak{G}$ and $p(\mathbf{V})$, then it is also not neurally identifiable through $\mathcal{V}^{\mathfrak{G}}$. Again, as in the proof for Cor.\ref{cor:cht}, the reverse direction in which we assume $Q$ to be identifiable analogously holds true for the same reason in this scenario.
\end{proof}

\section{Further Experimental Insights and Details}
The following subsections provide an extensive empirical analysis of causal effect estimation (\textbf{B.1}) and properties of the density estimation (\textbf{B.2}). 

\noindent{\bf Code.} We make our code publically available at \url{https://anonymous.4open.science/r/Relating-Graph-Neural-Networks-to-Structural-Causal-Models-A8EE}.

\subsection{Causal Effect Estimation}
Our experiments in the following investigate causal inference, namely causal effect estimation. We are interested in the average causal (or treatment) effect defined as $\text{ATE}(X,Y):=\mathbb{E}[Y{\mid}\doop(X{=}1)]-\mathbb{E}[Y{\mid}\doop(X{=}0)]$ \citep{pearl2009causality,peters2017elements}, where the binary variable $X$ is being referred to as treatment (and $Y$ is simply the effect e.g.\ recovery). Note that we can extend the ATE to be categorical/continuous, however, we focus on binary structures in the following, thereby the mentioned formulation is sufficient. Also, note that ATE is a special case of density estimation in which the same intervention location $X$ is being queried for the different intervention parameterizations $\doop(X=x)$, for binary variables this amounts to $\doop(X=i),i\in\{0,1\}$. General properties of the density estimation for the iVGAE are being systematically investigated in later sections of the Appendix (\textbf{B.2}, Figs.\ref{fig:invest-abc},\ref{fig:invest-def}). Nonetheless, it is an important sub-category since usually one is interested in causal inference rather than minuscule approximation of all derivable probability densities. In the following we do not consider identification, since the iVGAE is a data-dependent model (Def.\ref{def:dVGAE},Thm.\ref{thm:dualid}) it can not act as a proper SCM like NCM can (which are simply a neural net based reparameterization of the SCM formulation, and thus by construction capable of identification). In the following, Fig.\ref{fig:Density}(b) from the main paper will act as guiding reference to the subsequent analysis.

\paragraph{Considered SCM Structures.} We consider three different structural causal models structures (1: chain, 2: confounder, 3: backdoor) that are of significantly different nature in terms of information flow as dictated by the $d$-separation criterion \citep{koller2009probabilistic, pearl2009causality}. Fig.\ref{fig:Density}(b) bottom portraits the different structure (note the re-ordering of the variables, the graphs are being drawn in a planar manner). In the following we provide the exact parametric form for each of the SCMs $\mathcal{M}_i,i\in\{1..3\}$, the chain is given by
\begin{align}
\mathcal{M}_1 &= 
\left\{\begin{aligned}\quad
X & \leftarrow & f_X(U_X) = & U_X\\
Y & \leftarrow & f_Y(X, U_Y) = & X\wedge U_Y\\
Z & \leftarrow & f_Z(Y, U_Z) = & Y\wedge U_Z\\
W & \leftarrow & f_W(Z, U_W) = & Z\wedge U_W,\\
     \end{aligned}\right.
\end{align} the confounded structure is given by
\begin{align}
\mathcal{M}_2 &= 
\left\{\begin{aligned}\quad
X & \leftarrow & f_X(Z, U_X) = & Z\oplus U_X\\
Y & \leftarrow & f_Y(X,Z, U_Y) = & (X\wedge U_Y)\oplus (Z\wedge U_Y)\\
Z & \leftarrow & f_Z(U_Z) = & U_Z\\
W & \leftarrow & f_W(X, U_W) = & X\wedge U_W,\\
     \end{aligned}\right.
\end{align} and finally the backdoor structure is given by
\begin{align}
\mathcal{M}_3 &= 
\left\{\begin{aligned}\quad
X & \leftarrow & f_X(Z, U_X) = & Z \oplus U_X\\
Y & \leftarrow & f_Y(W, X, U_Y) = & X\wedge (W\wedge U_Y)\\
Z & \leftarrow & f_Z(U_Z) = & U_Z\\
W & \leftarrow & f_W(Z, U_W) = & Z\wedge U_W,\\
     \end{aligned}\right.
\end{align} where $\oplus$ denotes the logical XOR operation. We refer to logical operations to assert that the variables remain within $\{0,1\}$. Note that we consider Markovian SCM, thus the exogenous/noise terms are independent. We choose Bernoulli $\mathcal{B}(p),p\in [ 0, 1]$ distributions for the noise terms. We choose the $p_i$ for each of the terms $U_i$ uniformly at random to generate 5 different parameterizations of the same structure. For each intervention we create a data set of size 10000 and train a model consisting of two iVGAE modules. We consider three random seeds for each of the three parameterizations for each of the three structures, $3^3=27$ distinct optimizations. In the following we always consider the ATE of $X$ on $Y$, that is $Q=\text{ATE}(X,Y)$, which can be positive/negative $Q\neq 0$ or neutral $Q=0$ if their is neither a direct nor indirect influence from $X$ to $Y$. All ATE estimates we observed were approaching zero thus reasonable (more specifically, bounded in $[0,0.2]$ whereas the maximum error is $|\text{ATE}^*-\hat{\text{ATE}}|=|1-(-1)|=2$ i.e., the worst-case single approximation was off by $10\%$).

\paragraph{Interpretation for ATE Estimation on the chain structure $\mathcal{M}_1$.} The causal effect of $X$ on $Y$ is both direct and unconfounded. It is arguably the easiest structure to optimize, and the iVGAE correctly performs (see top/green row in Fig.\ref{fig:Density}(b)). The variance is further reduced in comparison, arguably due the relatively low variance in ground truth ATEs (since $\mathcal{M}_1$ ATE are mostly positive).

\paragraph{Interpretation for ATE Estimation on the confounder structure $\mathcal{M}_2$.} The causal effect of $X$ on $Y$ is direct yet confounded via $Z$. The ATE can thus obtain, given parameterizations, positive, negative and the zero value. The iVGAE reacts accordingly and is able to adequately estimate the causal effect. This is the key observation to causal inference since a correlation-based model would fail to return the correct answer i.e., it would simply return $p(Y{\mid}X)$ instead of $p(Y{\mid}\doop(X))$.

\paragraph{Interpretation for ATE Estimation on the backdoor structure $\mathcal{M}_3$.} The most difficult case since the causal effect of $X$ on $Y$ is confounded through a backdoor path $X\leftarrow Z \dots Y$. Nonetheless, the iVGAE adequately models the causal effect. The variance, like for $\mathcal{M}_2$, is increased which is arguably due to the spectrum of available ATE instantiations dependent on the concrete parameterization of the SCM.

\paragraph{Numerical Report.} In Tab.\ref{tab:identnumbers} we show numerical statistics on the trained models applied to the different SCM structures $\mathcal{M}_i$ for one of the three parameterizations averaged across seven random seeds. We show the two interventions on $X$ for computing the ATE, and report mean, best and worst ELBO (Def.\ref{eq:elbo}) and likelihood $\log p(x)$ performances (the higher the better). As expected, ELBO lower bounds the marginal log-likelihood and the validation performance, as desired, corresponds with the test performance.

\subsection{Systematic Investigation on Density Estimation}
We perform multiple experiments to answer various interesting questions. The following list enumerates all the key \textbf{questions} to be highlighted and discussed in this section:
\begin{enumerate}[(a)]
    \item What aspects of an interventional change through $\doop(\cdot)$ does the method capture?
    \item How does variance in ELBO \ref{eq:elbo} during variational optimization affect the method?
    \item When and how does the method fail to capture interventional distributions?
    \item At what degree does the performance of the method vary for different training durations?
    \item How does the method scale w.r.t.\ interventions while keeping capacity constant?
    \item How important is parameter tuning?
\end{enumerate}
For all the subsequent experiments we considered the same architecture. That is, an iVGAE (Def.\ref{def:iVGAE}) model consisting of two interventional GNNs (Def.\ref{def:iGNN}) for the encoder and decoder respectively where each GNN consists of 2 Sum-Pool Layers as introduced by \citep{kipf2016semi}. The decoder has $2B^2$ parameters, whereas the encoder has $3B^2$ parameters, where $B$ is the batch size, to allow for modelling the variance of the latent distribution. We consider datasets of size 10000 per intervention. The interventions are collected by modifying the data generating process of the data sets. For simplicity, we mostly consider uniform interventions. Optimization is done with RMSProp \citep{hinton2012neural} and the learning rate is set to 0.001 throughout. We perform a mean-field variational approximation using a Gaussian latent distribution and a Bernoulli distribution on the output. All data sets we have considered are binary but extensions to categorical or continous domains follow naturally. In the following, we focus on ASIA introduced in \citep{lauritzen1988local}, and Earthquake/Cancer covered within \citep{korb2010bayesian} respectively. We employ a training, validation and test set and use the validation set to optimize performance subsequently evaluated on the test set. We use a 80/10/10 split. We use 50 samples per importance sampling procedure to account for reproducibility in the estimated probabilities. Training is performed in 6000 base steps where each step considers batches of size $B$ that are being scaled multiplicatively with the number of interventional distributions to be learned. The adjacency provided to the GNNs is a directed acyclic graph (DAG) summed together with the identity matrix to allow for self-reference during the computation. The densities are acquired using an adjusted application of Alg.\ref{alg:inference}. All subsequent experiments are being performed on a MacBook Pro (13-inch, 2020, Four Thunderbolt 3 ports) laptop running a 2,3 GHz Quad-Core Intel Core i7 CPU with a 16 GB 3733 MHz LPDDR4X RAM on time scales ranging from a few minutes up to approximately an hour with increasing size of the experiments. In the following, Figs.\ref{fig:invest-abc} and \ref{fig:invest-def} act as reference for the subsequent subsections' elaborations on the questions (a)-(f). Numerical statistics are provided in Tab.\ref{tab:investnumbers}. For reproducibility and when reporting aggregated values (e.g. mean/median) we consider 10 random seeds. Our code is available at \url{https://anonymous.4open.science/r/Relating-Graph-Neural-Networks-to-Structural-Causal-Models-A8EE}.

\paragraph{Q-(a) What aspects of an interventional change through $\doop(\cdot)$ does the method capture?}
Consider Fig.\ref{fig:invest-abc}(a) in the following. It shows an iVGAE model trained on the observational ($\mathcal{L}_1$) and one interventional ($\mathcal{L}_2$, intervention $\doop(\text{tub}=\mathcal{B}(\frac{1}{2}))$) distributions, where the former is shown on the left and the latter on the right. We can observe that both the change within the intervention location (tub) but also in the subsequent change propagations along the causal sequence (either, xray, dysp) are being captured. In fact, they are being not only detected but also adequately modelled for this specific instance. If the optimization is successful in fitting the available data with the available model capacity, then this is the general observation we make across all the other settings we have evaluated i.e., the model can pick-up on the interventional change without restrictions.

\paragraph{Q-(b) How does variance in ELBO \ref{eq:elbo} during variational optimization affect the method?}
Consider Fig.\ref{fig:invest-abc}(b) in the following. Two different random seeds (that is, different initializations and thus optimization trajectories) for the same iVGAE under same settings (data, training time, etc.) are being shown. Clearly, the optimization for the seed illustrated on the left was successful in that the quantities of interest are being adequately estimated. However, the random seed shown on the right overestimates several variables (tub, either, xray) and simply does not fit as well. We argue that this is a general property of the variational method and ELBO (Eq.\ref{eq:elbo}) i.e., the optimization objective is non-convex and only a local optimum is guaranteed. Put differently, the variance in performance amongst random seeds (as measured by ELBO) is high.

\paragraph{Q-(c) When and how does the method fail to capture interventional distributions?}
Consider Fig.\ref{fig:invest-abc}(c) in the following. The predicted marginals of a single iVGAE model on the Earthquake dataset \citep{korb2010bayesian} are being presented for the observational density (right) and the interventional $\doop(\text{Earthquake}=\mathcal{B}(\frac{1}{2}))$ (left). The underlying graph in this real-world inspired data set is given by
\begin{align}
\begin{split}
    G = &(\{B,E,A,M,J\}, \\ &\{(B\rightarrow A), (E\rightarrow A), (A\rightarrow \{M,J\})\})
\end{split}
\end{align} where $B,E,A,M,J$ are Burglary, Earthquake, Alarm, MaryCalls and JohnCalls respectively. From $G$ we can deduce that the mutilated graph $G_I$ that is generated by the aforementioned Bernoulli-intervention $I=\doop(E=\mathcal{B}(\frac{1}{2}))$ will in fact be identical $G=G_I$. Put differently, conditioning and intervening are identical in this setting. The formulation for performing interventions in GNN (Def.\ref{def:iGNN}) only provides structural information i.e., information about the intervention location (and thus occurence) \textit{but not about the content of the intervention}. While this generality is beneficial in terms of assumptions placed onto the model, it also restricts the model in this special case where associational and internvetional distributions coincide. In a nutshell, computationally, the two posed queries $I_1=I$ and $I_2=\doop(\emptyset)$ are identical in this specific setting ($I_1=I_2$) and this is also being confirmed by the empirical result in Fig.\ref{fig:invest-abc}(c) i.e., the predictions are the same across all settings as follows naturally from the formulation in Def.\ref{def:iGNN} which in this case is a drawback. Generally, this insight needs to be considered a drawback of formulation Def.\ref{def:iGNN} opposed to being a  failure mode since the formulation indeed behaves as expected. In all our experiments, actual failure in capturing the densities seems to occur only in low model-capacity regimes, with early-stoppage or due to numerical instability.

\paragraph{Q-(d) At what degree does the performance of the method vary for different training durations?}
Consider Fig.\ref{fig:invest-def}(d) in the following. It shows the same model being probed for its predictions of the observational distributions at different time points, left early and right later (at convergence). Following intuition and expectation, training time does increase the performance of the model fit. Consider nodes tub and lung which were both underestimating in the earlier iterations while being perfectly fit upon convergence.

\paragraph{Q-(e) How does the method's scaling towards many interventions behave when keeping capacity constant?}
Consider Fig.\ref{fig:invest-def}(e) in the following. We show the same iVGAE model configurations being trained on either 2 interventional (top row) or 4 interventional distributions from the Earthquake dataset \citep{korb2010bayesian}. I.e., we keep the model capacity and the experimental settings consistent while increasing the difficulty of the learning/optimization problem by providing double the amount of distributions. As expected, we clearly see a degeneration in the quality of density estimation. The iVGAE model trained on 2 distributions adequately estimates the Bernoulli-interventional distribution $\doop(A=\mathcal{B}(\frac{1}{2}))$ (top right) while the model trained on more distributions (lower right) fails.

\paragraph{Q-(f) How important is parameter tuning?}
Consider Fig.\ref{fig:invest-def}(f) in the following. It shows an iVGAE model before (left) and (after) parameter tuning on the Bernoulli-intervention $\doop(\text{tub}=\mathcal{B}(\frac{1}{2}))$ on the ASIA dataset \citep{lauritzen1988local} where the tuned parameters involve aspects like pooling (sum, mean), layer numbers, learning rate, batch size etc. We clearly see a an improvement towards a perfect fit for certain nodes (smoke, tub, either). As expected, parameter tuning, as for any other machine learning model, is essential for improving the predictive performance. Especially, for universal density approximation, as discussed in Thm.\ref{thm:uda}, this is crucial - since in principle any density can be approximated with sufficient capacity and thus the dependence relies on the model itself but also on the optimization for optimality in practice.

\paragraph{Numerical Report.} In Tab.\ref{tab:investnumbers} we show numerical statistics on the trained models applied to the different data sets for answering the investigated questions (a)-(f). For reproducibility and stability, we performed 10 random seed variants per run. NaN values might occur for a single seed due to numerical instability in training, thus invalidating the whole run. We show the performance on different, increasing interventional data sets at various training iterations. We report mean, best and worst ELBO (Def.\ref{eq:elbo}) and likelihood $\log p(x)$ performances (the higher the better). Question (b) regarding the variance of ELBO becomes evident when considering the best-worst gaps. As expected, ELBO lower bounds the marginal log-likelihood. Also, by providing more distributions to learn, thus increasing difficulty, the quality of the fits in terms of ELBO/likelihood degenerates which is inline with what we observed regarding question (e). Finally, we might also note that the validation performance, as desired, corresponds with the test performance.

\paragraph{Full-width Figures and Tables.} Following this page are the figures and tables (Fig.\ref{fig:invest-abc}, \ref{fig:invest-def}, Tabs.\ref{tab:identnumbers}, \ref{tab:investnumbers}) that were referenced in the corresponding sections of the appendix.

\clearpage
\begin{figure*}[!ht]
\centering
\includegraphics[width=0.9\textwidth]{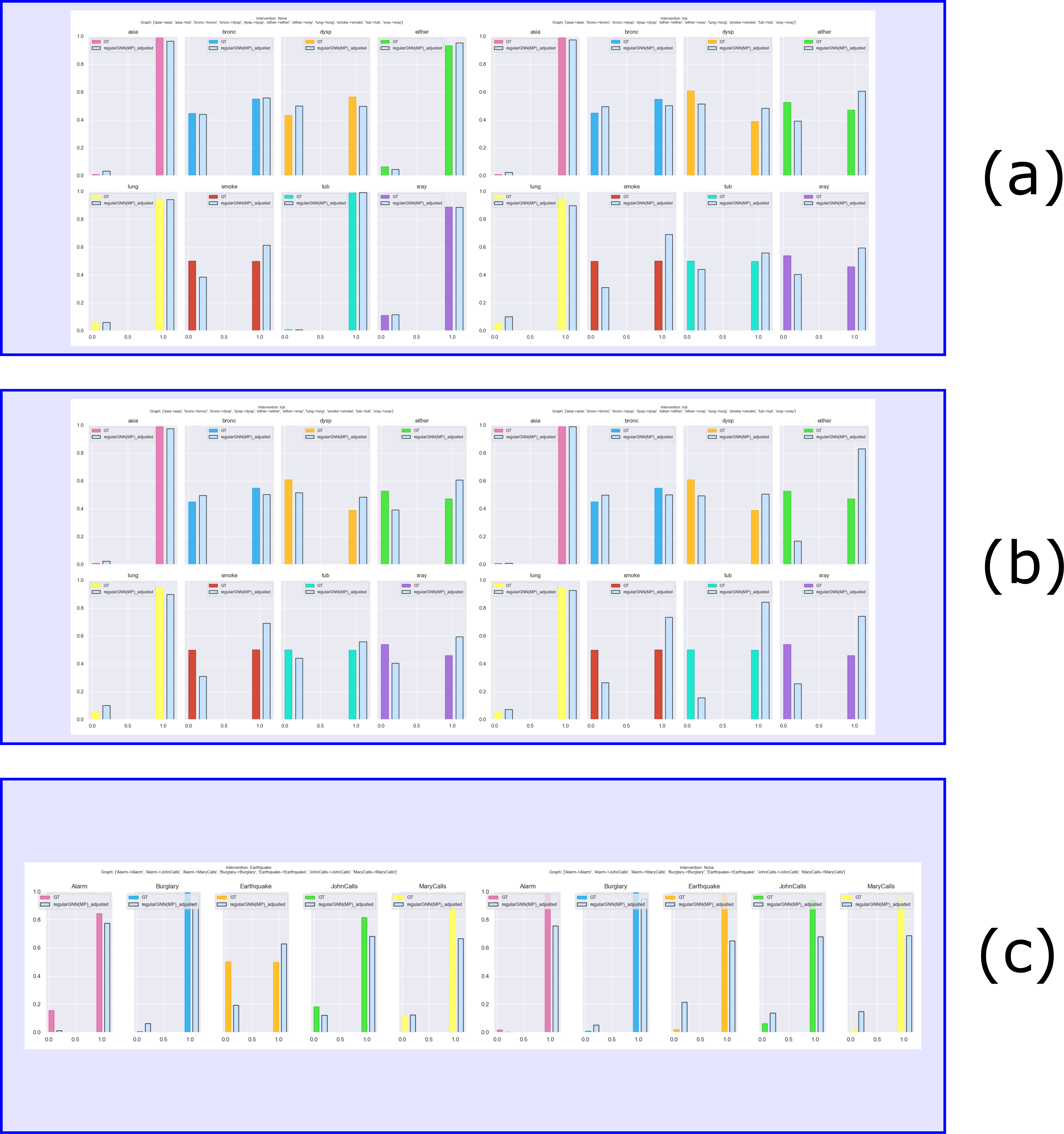}
\caption{\textbf{Systematic Investigation: Questions (a), (b), and (c).} The questions being answered by the presented illustrations in the respective blue box are (a) What aspects of an interventional change through $\doop(\cdot)$ does the method capture? (b) How does variance in ELBO \ref{eq:elbo} during variational optimization affect the method? and (c) When and how does the method fail to capture interventional distributions? Please consider the elaboration within the corresponding text segment. (Best viewed in color.)}
\label{fig:invest-abc}
\end{figure*}

\begin{figure*}[!ht]
\centering
\includegraphics[width=0.9\textwidth]{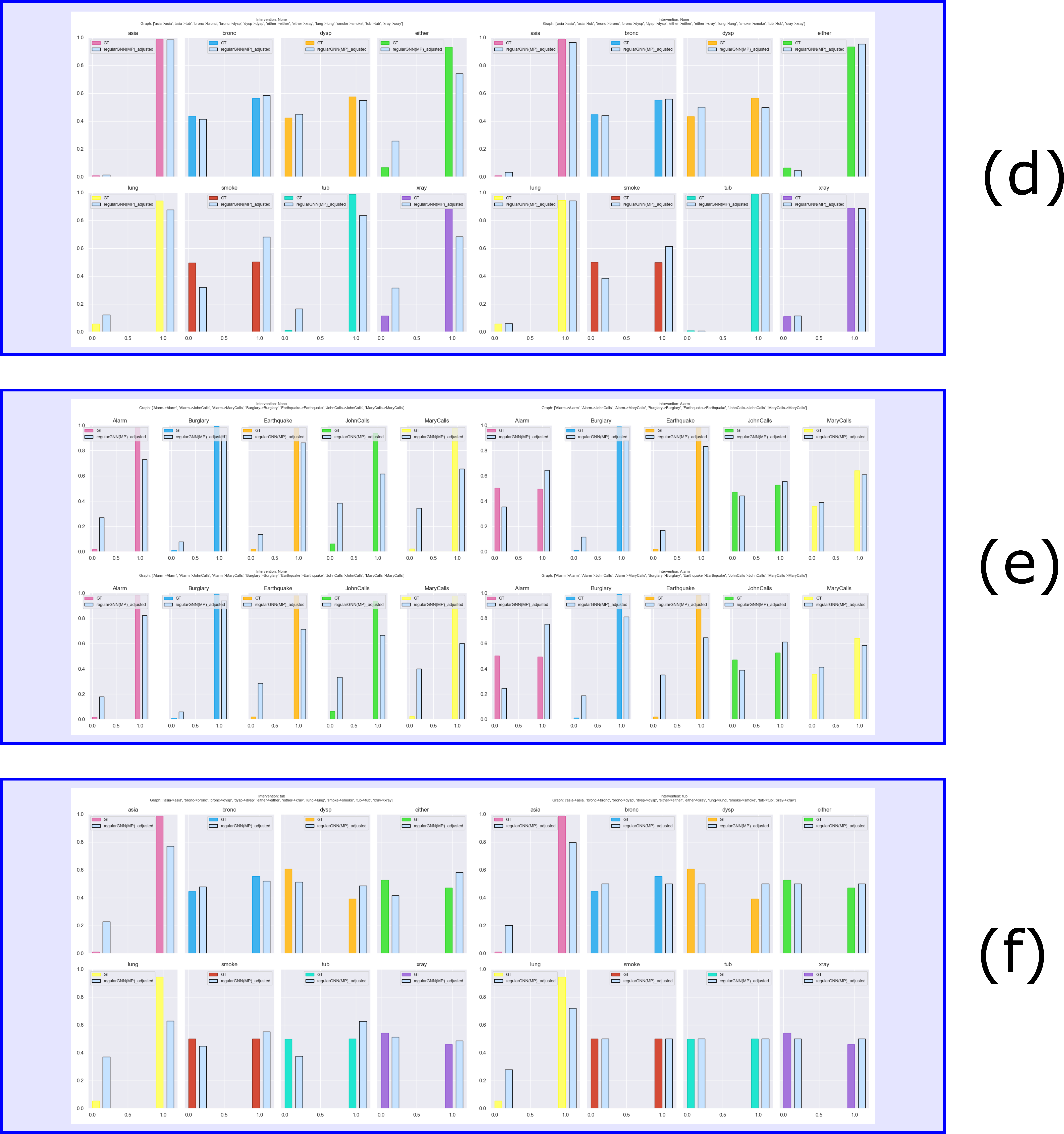}
\caption{\textbf{Systematic Investigation: Questions (d), (e), and (f).} The questions being answered by the presented illustrations in the respective blue box are (d) At what degree does the performance of the method vary for different training durations? (e) How does the method's scaling towards many interventions behave when keeping capacity constant? and (f) How important is parameter tuning? Please consider the elaboration within the corresponding text segment. (Best viewed in color.)}
\label{fig:invest-def}
\end{figure*}

\begin{table}[!t]
\begin{adjustbox}{max width=\textwidth}
\begin{tabular}{||c c c c c c c c c c||} 
 \hline
 $\mathfrak{C}_i$ & $\doop(X=x)$ & Steps & \thead{Mean \\ Train ELBO} &  \thead{Mean \\ Valid ELBO} &  \thead{Mean \\ Test ELBO} &  \thead{Mean \\ Valid $\log p(x)$} & \thead{Mean \\ Test $\log p(x)$} & \thead{Best \\ Test ELBO}  & \thead{Worst \\ Test ELBO} \\ [0.5ex] 
 \hline\hline
 1 & 1 & 2.5k & -1.48&       -1.60&    -1.59&          -1.43& -1.43 & -1.50 & -1.86\\
 1 & 0 & 2.5k & -1.86&         -2.12&       -2.07&            -1.85& -1.84&-1.70 & -2.47\\
 \hline
 2 & 1 & 2.5k &-1.48&      -1.60&    -1.59&        -1.43& -1.43& -1.37 & -1.86\\
 2 & 0 & 2.5k & -1.86&      -2.12&    -2.07&         -1.85& -1.84& -1.91 & -2.47\\
 \hline
 3 & 1 & 2.5k & -1.48&    -1.60&     -1.59&       -1.43& -1.43& -1.37 & -1.86\\
 3 & 0 & 2.5k &  -1.86&       -2.12&      -2.07&          -1.85  &-1.84& -1.91 & -2.47\\
 \hline
\end{tabular}
\end{adjustbox}
\caption{\textbf{Causal Effect Estimation: Key Statistics.} The aggregations cover three random seeds per model.}
\label{tab:identnumbers}
\end{table}

\begin{table*}[!t]
\begin{adjustbox}{max width=\textwidth}
\begin{tabular}{||c c c c c c c c c c||} 
 \hline
 Dataset & $|\mathcal{L}_2|$ & Steps & \thead{Mean \\ Train ELBO} &  \thead{Mean \\ Valid ELBO} &  \thead{Mean \\ Test ELBO} &  \thead{Mean \\ Valid $\log p(x)$} & \thead{Mean \\ Test $\log p(x)$} & \thead{Best \\ Test ELBO}  & \thead{Worst \\ Test ELBO} \\ [0.5ex] 
 \hline\hline
 ASIA & 2 & 16k & -3.43 & -6.02 & -4.60 & -4.15 & -4.11 & -4.10 & -5.37\\ 
 \hline
 ASIA & 4 & 16k & NaN  &            NaN  &     -4.61            &      NaN  & -4.05 & -3.79 & -5.59 \\
 \hline
 Cancer & 2 & 12k & -2.26&  -4.76&       -3.17&          -3.66& -2.76& -2.35 & -4.49
 \\
 \hline
 Cancer & 4 & 12k &  NaN    &          NaN &      -3.26&               NaN&  -2.88& -2.43 & -4.53
 \\
 \hline
 Earthquake & 2 & 12k & -1.21&     -3.02&      -2.43&         -1.92& -1.93 & -1.49 & -3.50\\
 \hline
 Earthquake & 4 & 12k & -0.78&      -4.67&       -2.77&       -2.31& -2.27 & -1.75 & -3.46 \\  
 \hline
\end{tabular}
\end{adjustbox}
\caption{\textbf{Density Estimation: Key Statistics.} The aggregations cover 10 random seeds for each of the models respectively.}
\label{tab:investnumbers}
\end{table*}

\end{document}